%% file: main.tex
\documentclass[11pt]{article}

\usepackage{microtype}
\usepackage{graphicx}
\usepackage{subfigure}
\usepackage{booktabs} 

\usepackage{mathtools}

\usepackage{fullpage}
\usepackage{mathrsfs}

\usepackage{authblk}
\usepackage{smile}
\usepackage[usenames]{color}
\usepackage[breaklinks,colorlinks,
linkcolor=red,
anchorcolor=blue,
citecolor=blue
]{hyperref}

\usepackage[OT1]{fontenc}

\def\rf{\rm{F}}
\def\tcD{{\tilde{\mathcal{D}}}}
\def\rd{{\mathrm{d}}}
\def\tcE{{\tilde{\mathcal{E}}}}
\def\msbe{{\mathrm{MSBE}}}
\def\Vskip{\vskip4pt}

\def\tP{\tilde{P}}

\def\tPhi{\tilde{\Phi}}
\def\tphi{\tilde{\phi}}

\def\d|{\,\|\,}
\def\sF{\mathscr{F}}
\def\tpsi{{\tilde{\psi}}}
\def\bepsilon{{\bar{\epsilon}}}
\def\PPE{\frac 1T\int_0^T \norm{Q_t-Q^{\pi_t}}_{2,\tphi^{\pi_t}}\rd t}
\def\tW{\tilde{W}}
\def\talpha{\tilde \alpha}
\def\tbeta{\tilde \beta}
\def\nend{\nonumber\\}
\def\tf{\tilde{f}}
\def\tcP{\tilde{\cP}}
\def\sP{\mathscr{P}}

\def\lip{{\rm Lip}}

\long\def\comment#1{}

\DeclarePairedDelimiterX{\norm}[1]{\|}{\|}{#1}

\DeclarePairedDelimiterX{\inp}[2]{\langle}{\rangle}{#1, #2}

\def\msbe{\text{MSBE}}

\def\Div{\mathop{\mathrm{div}}}

\def\law{\mathop{\mathrm{law}}}

\newcommand{\kl}{{\mathrm{KL}}}

\newcommand\Hnorm[3][]{\norm[#1]{#2}_{\dot H^{-1}(#3)}}
\newcommand\normf[2][]{#1| #2 #1|}
\newcommand\Cnorm[4][]{\chi^2#1(#3\,#2\|\,#4#1)}

\newcommand\extrafootertext[1]{%
    \bgroup
    \renewcommand\thefootnote{\fnsymbol{footnote}}%
    \renewcommand\thempfootnote{\fnsymbol{mpfootnote}}%
    \footnotetext[0]{#1}%
    \egroup
}

\title{\huge \textbf{Wasserstein Flow Meets Replicator Dynamics: A Mean-Field Analysis of Representation Learning in Actor-Critic}}

\author[1,$\dagger$]{Yufeng Zhang}

\author[2,$\dagger$]{Siyu Chen}

\author[3]{Zhuoran Yang}

\author[4]{Michael I. Jordan}

\author[1]{Zhaoran Wang}

\affil[1]{\footnotesize \textit{Department of Industrial Engineering and Management Sciences, Northwestern University}}
\affil[2]{\footnotesize \textit{Department of Statistics and Data Science, Yale University}}
\affil[3]{\footnotesize \textit{Department of Operations Research and Financial Engineering, Princeton University}}
\affil[4]{\footnotesize \textit{Department of EECS and Statistics, University of California, Berkeley}}

\begin{document}
	
	\maketitle
	\setlength\abovedisplayskip{4pt}
	\setlength\belowdisplayskip{4pt}

\extrafootertext{$^\dagger$ Equal contribution.
}

	\input{intro}
	\input{bg}
	\input{algo}
	\input{result}

\section{Acknowledgement}
Zhaoran Wang acknowledges National Science Foundation (Awards 2048075, 2008827, 2015568, 1934931), Simons Institute (Theory of Reinforcement Learning), Amazon, J.P. Morgan, and Two Sigma for their supports.

	\bibliographystyle{ims}
	\bibliography{graphbib,rl_ref}
	
	\newpage
	\appendix
	\input{w_2}
	\input{theory}

	\input{appendix}

	\input{Limitation}
	
\end{document}

%% file: intro.tex

\begin{abstract}
	
	Actor-critic  (AC) algorithms, empowered by neural networks, have had significant empirical success in recent years. 
	However, most of the existing theoretical support for AC algorithms focuses on the case of linear function approximations, or linearized neural networks, where the feature representation is fixed throughout training. 
	Such a limitation fails to capture the key aspect of representation learning in neural AC, which is pivotal in practical problems.
	In this work, we take a mean-field perspective on the evolution and convergence of feature-based neural AC. 
	Specifically, we consider a version of  AC where 
	the actor and critic are represented by overparameterized two-layer neural networks and are updated with two-timescale learning rates. 
	The critic is updated by temporal-difference (TD) learning with a larger stepsize while the actor is updated via proximal policy optimization (PPO) with a smaller stepsize. 
	In the continuous-time and infinite-width  limiting regime, when the timescales are properly separated,
	we prove that neural AC finds the globally optimal policy at a sublinear rate. 
	Additionally,  we prove that the feature representation induced by the critic network is allowed to evolve within a neighborhood of the initial one. 
\end{abstract}

\section{Introduction}

In reinforcement learning (RL) \citep{sutton2018reinforcement}, an agent aims to learn the optimal policy that maximizes the expected total reward obtained from interactions with an environment.
Policy-based RL algorithms achieve such a goal by directly optimizing the expected total reward as a function of the policy, which often involves two components: policy evaluation and policy improvement.
Specifically,  policy evaluation refers to estimating the value function of the current policy, which characterizes the performance of the current policy and reveals the updating direction for finding a better policy, which is known as policy improvement. 
Algorithms that explicitly incorporate these two ingredients are called actor-critic (AC) methods \citep{konda2000actor}, 
where the actor and the critic refer to the policy and its corresponding value function, respectively. 

Recently, in RL applications with large state spaces, 
actor-critic methods have achieved striking empirical successes when empowered by expressive function approximators such as neural networks \citep{agostinelli2019solving,akkaya2019solving,berner2019dota,
	duan2016benchmarking,silver2016mastering, silver2017mastering,  vinyals2019grandmaster}.
These successes benefit from the 
data-dependent representations learned by neural networks.  Unfortunately, however, the theoretical understanding of this data-dependent benefit is very limited. 
The classical theory of AC focuses on the case of linear function approximation, where the actor and critic are represented using linear functions with the feature mapping fixed throughout learning \citep{bhatnagar2008incremental, bhatnagar2009natural,konda2000actor}. 
Meanwhile,  a few recent works establish convergence and optimality of AC with overparameterized neural networks \citep{fu2020single,liu2019neural,wang2019neural},
where the neural network training is captured by the Neural Tangent Kernel (NTK)  \citep{jacot2018neural}.
Specifically, with properly designed  parameter  initialization 
and stepsizes, and sufficiently large network widths, the neural networks employed by both actor and critic can be assumed to be well approximated by linear functions of a random feature vector. 
In other words, from the point of view of representation learning, 
the features induced by these algorithms are by assumption infinitesimally close to the initial featural representation, which is data-independent. 

In this work, we make initial steps towards understanding how representation learning comes into play in neural AC. 
Specifically, we address the following questions:
\begin{center}
	{\it Going beyond the NTK regime, does neural AC provably find the globally optimal policy? 
	How does the feature representation associated with the neural network evolve along with neural AC?}
\end{center}

We focus on a version of AC where the critic performs temporal-difference (TD) learning \citep{sutton1988learning} for policy evaluation and the actor improves its policy via proximal policy optimization (PPO) \citep{schulman2017proximal}, 
which corresponds to a Kullback-Leibler (KL) divergence regularized optimization problem, with the critic providing the update direction. 
Moreover, 
we utilize two-timescale updates where both the actor and critic are updated at each iteration but with the critic having a much larger stepsize. 
In other words, the critic is updated at a faster timescale. 
Meanwhile, 
we represent the critic explicitly as a  two-layer overparameterized neural network and parameterize the actor implicitly via the critic and PPO updates. 
To examine convergence, we study the evolution of the actor and critic in a continuous-time limit with the network width going to infinity. 
In such a regime, 
the actor update is closely connected to replicator dynamics \citep{ borgers1997learning, hennes2020neural,schuster1983replicator} and the critic update is captured by a semigradient flow in the Wasserstein space \citep{villani2008optimal}. 
Moreover, the semigradient flow runs at a faster timescale according to the two-timescale mechanism. 

It turns out that the separation of timescales plays an important role in the convergence analysis. 
In particular, the continuous-time limit 
enables us to separately analyze the evolution of actor and critic and then combine these results to get final theoretical guarantees.
Specifically, focusing solely  on  the actor, we prove that 
the time-averaged suboptimality of the  
actor converges sublinearly to zero up to the time-averaged policy evaluation error associated with critic updates. 
Moreover, for the critic, 
under regularity conditions, 
we 
connect the Bellman error 
to the Wasserstein distance and show 
that the time-averaged policy evaluation error also converges sublinearly to zero. 
Therefore, we show that neural AC provably achieves global optimality at a  sublinear rate. 
Furthermore, regarding representation learning,  we show that the  critic   induces a data-dependent feature representation within an  $O(1/\alpha)$ neighborhood 
of the initial representation in terms of the Wasserstein distance, where $\alpha$ is a sufficiently large scaling parameter.

The key to our technical analysis reposes on three ingredients: (i) infinite-dimensional variational inequalities with a one-point monotonicity \citep{harker1990finite}, (ii) a mean-field perspective on neural networks \citep{chizat2018global, mei2019mean, mei2018mean,sirignano2020mean1, sirignano2020mean},  and (iii) the two-timescale stochastic approximation \citep{borkar2009stochastic, kushner2003stochastic}. 
In particular, in the infinite-width limit, the neural network and its induced feature representation are identified with a distribution over the parameter space. 
The mean-field perspective enables us to characterize the evolution of 
such a distribution within the Wasserstein space via a certain partial differential equation (PDE) \citep{ambrosio2013user, ambrosio2008gradient, villani2003topics, villani2008optimal}.  
For policy evaluation, this PDE is given by the semigradient flow induced by TD learning. 
We characterize the error of policy evaluation by showing that mean-field Bellman error satisfies a version of one-point monotonicity tailored to the Wasserstein space. 
Moreover, our actor analysis utilizes the geometry of policy optimization, which shows that the expected total reward, as a function of the policy, also enjoys the property of one-point monotonicity in the policy space. 
Finally, the actor and critic errors are connected via two-timescale stochastic approximation. 
To the best of our knowledge, this is the first time that convergence and global optimality guarantees have been obtained for neural AC.

{\noindent \bf Related Work.} 
AC with linear function approximation has been studied extensively in the literature. 
In particular, using a two-timescale stochastic approximation via ordinary differential equations,  \cite{konda2000actor,bhatnagar2008incremental} and \cite{bhatnagar2009natural} establish asymptotic convergence guarantees in the continuous-time limiting regime. 
More recently, 
using more sophisticated  optimization techniques, 
various works 
\citep{wu2020finite, xu2020non,xu2020improving, hong2020two, khodadadian2021finite} have established discrete-time convergence guarantees that show that   linear AC 
converges sublinearly to either a stationary point or the globally optimal policy. 
Furthermore, when overparameterized neural networks are employed, 
\cite{liu2019neural,wang2019neural} and \cite{fu2020single} prove that neural AC converges to the global optimum at a sublinear rate. 
In these works, 
the initial value of the network parameters and the learning rates are chosen such that both actor and critic updates are captured by the NTK. 
In other words, when the network width is sufficiently large, 
such a version of neural AC is well approximated by its linear counterpart via the neural tangent feature. 
In comparison, we establish a mean-field analysis that has a different scaling than the NTK regime. 
We also establish finite-time convergence to global optimality, and more importantly, the feature representation induced by the critic is data-dependent and allowed to evolve within a much larger neighborhood around the initial one. 

Our work is also related to a recent line of research on understanding stochastic gradient descent (SGD) for supervised learning problems involving an overparameterized two-layer neural network under the mean-field regime. 
See, e.g., \cite{chizat2018global, mei2018mean, mei2019mean, javanmard2019analysis, wei2019regularization, fang2019convex, fang2019over, chen2020computation, sirignano2020mean, sirignano2020mean1, lu2020mean} and the references therein. 
In the continuous-time and infinite-width limit, 
these works show that SGD for neural network training is captured by a Wasserstein gradient flow \citep{villani2008optimal, ambrosio2008gradient, ambrosio2013user} of an energy function that corresponds to the objective function in supervised learning. 
In contrast, our analysis combines such a  mean-field analysis with TD learning and two-timescale stochastic approximation, which are tailored specifically to AC. 
Moreover, our critic is updated via TD learning, which is a semigradient algorithm and there is no objective functional making TD learning a gradient-based algorithm. Thus, in the mean-field regime, our critic is given by a Wasserstein semigradient flow, which also differs from these existing works.  

Additionally, our work is closely related to \cite{zhang2020can} and \cite{agazzi2019temporal}, who provide mean-field analyses for neural TD-learning and Q-learning \citep{watkins1992q}. 
In comparison, we focus on AC, which is a two-timescale policy optimization algorithm.  
Finally, \cite{agazzi2020global} study softmax policy gradient with neural network policies in the mean-field regime, 
where policy gradient is cast as a Wasserstein gradient flow with respect to the expected total reward.  
The algorithm assumes that the critic directly gets the desired value function and thus the algorithm is single-timescale. 
Moreover, the convergence guarantee in \cite{agazzi2020global} is asymptotic. 
In comparison, our AC is two-timescale and we establish non-asymptotic sublinear convergence guarantees to global optimality.

\noindent{\bf Notation.} 
We denote by $\sP(\cX)$ the set of probability measures over the measurable space $\cX$.
Given a curve $\rho:\RR\rightarrow \cX$, we denote by $\dot\rho_s=\partial_t \rho_t\given_{t=s}$ its derivative with respect to time.
For an operator $F: \cX \rightarrow \cX$ and a measure $\mu \in \sP(\cX)$, we denote by $F_\sharp \mu = \mu \circ F^{-1}$ the push forward of $\mu$ through $F$.
We denote by $\chi^2(\rho\d|\mu)$ the chi-squared divergence between probability measures $\rho$ and $\mu$, which is defined as $\chi^2(\rho \d| \mu)=\int (\rho/\mu-1)^2 \rd\mu$.
Given two probability measures $\rho$ and $\mu$, we denote the Kullback-Leibler divergence or the relative entropy from $\mu$ to $\rho$ by $\kl(\rho \d| \mu)=\int \log(\rho/\mu)\rd \rho$.
For $\nu_1,\nu_2, \mu\in \sP(\cX)$, we define the $\dot H^{-1}(\mu)$ weighted homogeneous Sobolev norm as
$\Hnorm{\nu_1-\nu_2}{\mu}=\sup\big\{|\inp{f}{\nu_1-\nu_2}| \big| \norm{f}_{\dot{H}^1(\mu)}\le 1\big\}
$.
We denote by $\norm{f(x)}_{p,\mu}=(\int |f(x)|^p \mu(\rd x))^{1/p}$ the $\ell_p$-norm with respect to probability measure $\mu$.
We denote by $\otimes$ the semidirect product, i.e., $\mu \otimes K=K(y\given x)\mu(x)$ for $\mu \in \sP(\cX)$ and transition kernel $K:\cX\rightarrow \sP(\cY)$.
For a function $f: \cX \rightarrow \RR$, we denote by $\lip(f) = \sup_{x, y\in \cX, x\neq y} |f(x) - f(y)| / \norm{x- y}$ its Lipschitz constant.
We denote a normal distribution on $\RR^D$ by $\cN(\mu, \Sigma)$, where $\mu$ is the mean value and $\Sigma$ is the covariance matrix.

%% file: bg.tex
\section{Background}
In this section, we first introduce the policy optimization problem and the actor-critic method.
We then present the definition of the Wasserstein space.
\subsection{Policy Optimization and Actor-Critic}
We consider a Markov decision process (MDP) given by $ (\cS, \cA, \gamma, P, r, \cD_0) $, where $\cS \subseteq \RR^{d_1}$ is the state, $\cA\subseteq \RR^{d_2}$ is the action space, $\gamma \in (0, 1)$ is the discount factor, $P: \cS\times \cA\rightarrow \sP(\cS)$ is the transition kernel, $r: \cS\times \cA \rightarrow \RR_+$ is the reward function, and $\cD_0 \in \sP(\cS)$ is the initial state distribution. 
Without loss of generality, we assume that $\cS\times \cA \subseteq \RR^{d}$ and $\norm{(s, a)}_2 \le 1$, where $d=d_1+d_2$. 
 We remark that as long as the state-action space is bounded, we can normalize the space to be within the unit sphere.
Given a policy $\pi: \cS\times \cA \rightarrow \sP(\cS)$, at the $m$th step, the agent takes an action $a_m$ at state $s_m$ according to $\pi(\cdot\given s_m)$ and observes a reward $r_m = r(s_m, a_m)$. 
The environment then transits to the next state $s_{m+1}$ according to the transition kernel $P(\cdot\given s_m, a_m)$. 
Note that the policy $\pi$ induces Markov chains on both $\cS$ and $\cS\times\cA$.
Considering the Markov chain on $\cS$, we denote the induced Markov transition kernel by $P^\pi: \cS \rightarrow \sP(\cS)$, which is defined as $P^\pi(s' \given s) = \int_\cA P(s' \given s, a) \pi( \rd a\given s) $. 
Likewise, we denote the Markov transition kernel on $\cS\times \cA$ by $ \tP^\pi:\cS\times\cA \rightarrow \sP(\cS\times \cA)$, which is defined as $\tP^\pi(s',a'\given s,a)=\pi(a'\given s')P(s'\given s,a)$.
Let $\tcD$ be a probability measure on $\cS\times\cA$.
We then define the visitation measure induced by policy $\pi$ and starting from $\tcD$ as
\begin{align}
	\label{eq:visitation}
	\tcE^\pi_{\tcD}\big(\rd (s,a)\big) = (1-\gamma) \cdot \sum_{m\ge 0} \gamma^m \cdot \PP\big((s_m,a_m) \in \rd (s,a)\given(s_0,a_0)\sim \tcD\big),
\end{align}
where $(s_m,a_m)$ is the trajectory generated by starting from $(s_0,a_0) \sim \tcD$ and following policy $\pi$ thereafter. 
If $\tcD=\cD\otimes\pi$ holds, we then denote such a visitation measure by $\tilde \cE_{\cD}^\pi$.
Furthermore, we denote by $\cE(\rd s)=\int_{\cA} \tcE (\rd s, \rd a)$ the marginal distribution of visitation measure $\tcE$ with respect to $\cS$. 
In particular, when $(s_0,a_0)\sim \cD\otimes \pi$ holds in \eqref{eq:visitation}, it follows that $\tcE_{\cD}^{\pi}=\cE_{\cD}^{\pi}\otimes\pi$.
In policy optimization, we aim to maximize the expected total reward  $J(\pi)$ defined as follows,
\begin{align*}
	J(\pi) = \EE^{\pi}\biggl[ \sum_{m\ge 0} \gamma^m \cdot r (s_m, a_m) \biggiven s_0 \sim \cD_0 \biggr],
\end{align*}
where we denote by $\EE^\pi$ the expectation with respect to $a_m \sim \pi(\cdot\given s_m)$ and $s_{m+1} \sim P(\cdot \given s_m, a_m)$ for $m\ge 0$. We define   the action value function $Q^\pi: \cS\times \cA \rightarrow \RR$ and the state value function $V^\pi : \cS\rightarrow \RR$ as follows,
\begin{align}\label{eq:def-q}
	Q^\pi(s, a) = \EE^\pi\biggl[ \sum_{m\ge 0} \gamma^m \cdot r (s_m, a_m) \biggiven s_0 = s, a_0 = a \biggr], \quad V^\pi(s) = \inp[\big]{Q^\pi(s, \cdot)}{\pi(\cdot \given s)}_\cA,
\end{align}
where we denote by $\inp{\cdot}{\cdot}_{\cA}$ the inner product on the action space $\cA$.
Correspondingly, the advantage function $A^\pi: \cS\times \cA \rightarrow \RR$ is defined as
\begin{align*}
	A^\pi(s, a) = Q^\pi(s, a) - V^\pi(s).
\end{align*}
The action value function $Q^\pi$ corresponds to the minimizer of the following mean-squared Bellman error (MSBE),
\begin{align}
	\label{eq:msbe}
	\mathrm{MSBE}(Q;\pi) = \frac{1}{2} \EE_{(s,a)\sim \tilde \Phi^{\pi}} \Bigl[ \bigl( Q(s, a) - r(s, a) - \gamma \EE_{(s',a')\sim \tP^\pi(\cdot\given s,a)} [Q(s', a')]
	\bigr)^2 \Bigr],
\end{align}
where $\tilde \Phi^\pi$ is the sampling distribution depending on policy $\pi$, which will be defined by \eqref{eq:tPhi-construct} in \S\ref{sec:optimality of AC}. Note that when $\tilde \Phi^\pi$ is with full support, i.e., $\mathrm{supp}(\tilde \Phi^\pi)=\cS\times\cA$, $Q^\pi$ is the unique global minimizer to the MSBE.
Consequently, the policy optimization problem can be written as the following bilevel optimization problem,
\begin{align}
	\label{eq:bilevel}
	&\max_\pi J(\pi) = \EE_{s \sim \cD_0} \Bigl[ \inp[\big ]{Q^\pi(s, \cdot)}{\pi(\cdot\given s)}_\cA \Bigr],  \quad \text{subject to }Q^\pi = \argmin_Q \msbe(Q;\pi).  
\end{align}
The inner problem in \eqref{eq:bilevel} is known as a policy evaluation subproblem, while the outer problem is the policy improvement subproblem. One of the most popular way to solve the policy optimization problem is actor-critic (AC) methods \citep{sutton2018reinforcement}, where the job of the critic is to evaluate current policy and then the actor updates its policy according to the critic's evaluation.

\subsection{Wasserstein Space}\label{sec:wasserstein space}
Let $\Theta \subseteq \RR^D$ be a Polish space. We denote by $\sP_2(\Theta) \subseteq \sP(\Theta)$ the set of probability measures with finite second moments. Then, the Wasserstein-2 distance between $\mu, \nu \in \sP_2(\Theta)$ is defined as follows,
\begin{align*}
	W_2(\mu, \nu) = \inf\Bigl\{ \EE\bigl[\norm{X - Y}^2\bigr]^{1/2} \Biggiven \law (X) = \mu, \law (Y) = \nu  \Bigr\},
\end{align*} 
where the infimum is taken over the random variables $X$ and $Y$ on $\Theta$ and we denote by ${\rm law}(X)$ the distribution of a random variable $X$.
We call $\cM = (\sP_2(\Theta), W_2)$ the Wasserstein ($W_2$) space, which is an infinite-dimensional manifold \citep{villani2008optimal}. See \S\ref{append:wasserstein space} for more details.

%% file: algo.tex
\section{Algorithm}\label{sec:Algorithm}
\noindent\textbf{Two-timescale actor-critic.}
We consider a two-timescale actor-critic (AC) algorithm \citep{kakade2002natural,peters2008natural} for the policy optimization problem in \eqref{eq:bilevel}.
For policy evaluation, we parameterize the critic $Q$ with a neural network and update the parameter via temporal-difference (TD) learning \citep{sutton1988learning}.
For policy improvement, we update the actor policy $\pi$ via proximal policy optimization (PPO) \citep{schulman2017proximal}. 
Our algorithm is two-timescale since both the actor and critic are updated at each iteration with different stepsizes.  
Specifically, we parameterize the critic $Q$ by the following neural network with width $M$ and parameter $ \hat{\bm\theta} = (\hat \theta^{(1)}, \cdots,\hat \theta^{(M)})\in \RR^{D\times M}$,
\begin{align}\label{eq:Q-para}
	Q_{\hat{\bm \theta}}(s, a) = \frac{\alpha}{M} \sum_{i=1}^{M} \sigma(s,a;\hat \theta^{(i)}).
\end{align}
Here $\sigma(s, a;\theta): \cS\times \cA \times \RR^D \rightarrow \RR$ is the activation function and $\alpha>0$ is the scaling parameter.
Such a structure also appears in \cite{chizat2018note,mei2019mean} and \cite{chen2020mean}.
In a descrete-time finite-width (DF) scenario, at the $k$th iteration, the critic and actor are updated as follows,
\begin{align}
	&\text{DF-TD:}\quad\hat\theta_{k+1}^{(i)} = \hat\theta_k^{(i)} -  \frac{\varepsilon'}{ \alpha}  \bigl(Q_{\hat{\bm{\theta}}_k}(s_k, a_k) - r(s_k, a_k) - \gamma Q_{\hat{\bm\theta}_k}(s'_k, a'_k) \bigr) \nabla_{\theta} \sigma(s, a;\hat\theta_k^{(i)}),\label{eq:td-discrete}\\
	&\text{DF-PPO:}\quad\hat\pi_{k+1}(\cdot\given s) = \argmax_{\pi}
	\Bigl\{ \inp[\big]{Q_{\hat{\bm\theta}_k}(s, \cdot)}{\pi(\cdot\given s) }_\cA - \varepsilon^{-1} \cdot \kl\bigl(\pi(\cdot\given s) \,\|\,\hat\pi_k(\cdot \given s) \bigr) \Bigr\}
	\label{eq:ppo-discrete},
\end{align}
where $(s_k,a_k) \sim \tPhi^{\hat \pi_k}$ and $(s'_k,a'_k)\sim \tP^{\hat\pi_k}(\cdot \given s_k,a_k)$.
Here $\hat \pi_k$ is the policy for the actor at the $k$th iteration, $\tPhi^{\hat \pi_k}$ is the corresponding weighting distribution, $\varepsilon$ and $\varepsilon'$ are the stepsizes for the DF-PPO update and the DF-TD update, respectively. In \eqref{eq:td-discrete}, the scaling of $\alpha^{-1}$ arises since our update falls into the lazy-training regime \citep{chizat2018note}. In the sequel, we denote by $\eta=\varepsilon'/\varepsilon$ the relative TD timescale.
Note that in a double-loop AC algorithm, the critic can usually be solved with high precision.
In the two-timescale AC however, even with the KL-divergence term in \eqref{eq:ppo-discrete} which regularizes the policy update and helps to improve the local estimation quality of the TD update,  the critic $Q_{\hat{\bm\theta}_k}$ for updating the actor's policy $\hat\pi_k$  can still be far from the true action value function $Q^{\hat\pi_k}$. Since the policy evaluation problem is not fully solved at each iteration, the two-timescale AC can be more efficient in computation while more challenging to characterize theoretically.
\Vskip

\noindent\textbf{Mean-field (MF) Limit.} To analyze the convergence of the two-timescale AC with neural networks, we employ an analysis that studies the mean-field limit regime \citep{mei2018mean,mei2019mean}. Here, by saying the mean-field limit, we refer to an infinite-width limit, i.e., $M\rightarrow \infty$ for the neural network width $M$ in \eqref{eq:Q-para}, and a continuous-time limit, i.e., $t=k\varepsilon$ where $\varepsilon\rightarrow 0$ for the stepsize in \eqref{eq:td-discrete} and \eqref{eq:ppo-discrete}. 
For $\hat{\bm \theta}=\{\hat \theta^{(i)}\}_{i=1}^{M}$ independently sampled from a distribution $\rho$, we can write the infinite-width limit of \eqref{eq:Q-para} as
\begin{align}\label{eq:nn-inf}
	Q(s,a;\rho)=\alpha \int \sigma(s,a;\theta) \rho(\rd \theta).
\end{align}
In the sequel, we denote by $\hat \rho_k$ the distribution of $\hat \theta_k^{(i)}$ for the infinite-width limit of the neural network at the $k$th iteration. We further let $\rho_t$ and $\pi_t$ be the continuous-time limits of $\hat \rho_k$ and $\hat \pi_k$, respectively. The existence of $\pi_t$ will be shown shortly after. As derived in \cite{zhang2020can}, the mean-field limit of the DF-TD update in \eqref{eq:td-discrete} is
\begin{align} \label{eq:td-cont}
	\text{MF-TD:}\quad\partial_t \rho_t = -\eta \Div\bigl( \rho_t \cdot g(\cdot; \rho_t, \pi_t)\bigr),
\end{align}
where $\eta$ is the relative TD timescale and
\begin{align}
	\label{eq:g-rho}
	g(\theta; \rho, \pi) = - \EE_{\tPhi^\pi}^\pi\Bigl\{ \big[Q(s, a; \rho) - r(s, a) - \gamma \cdot Q(s', a'; \rho)\big] \cdot \alpha^{-1}\nabla_\theta \sigma(s, a; \theta) \Bigr\}
\end{align}
is a vector field.
Here $\EE_{\tPhi^\pi}^{\pi}$ is taken with respect to $(s,a)\sim \tPhi^\pi$ and $(s',a')\sim \tP^\pi(\cdot\given s,a)$.
It remains to characterize the mean-field limit of the DF-PPO update in \eqref{eq:ppo-discrete}. By solving the maximization problem in \eqref{eq:ppo-discrete}, the infinite-width limit of DF-PPO update can be written in closed form as
\begin{align*}
	\varepsilon^{-1} \cdot \Big\{
	\log\big[\hat \pi_{k+1}(a\given s)\big]
	- \log\big[\hat\pi_k(a\given s)\big] 
	\Big\} =  Q(s,a;\hat\rho_k) - \hat Z_k(s),
\end{align*}
where $\hat Z_k(s)$ is the normalizing factor such that $\int \hat \pi_k(\rd a \given s) = 1$ for any $s \in \cS$. 
By letting $t=k\varepsilon$ and $\varepsilon \rightarrow 0$, such a result directly shows that the limit of $\hat \pi_k$ exists. Therefore, in the mean-field limit, we have $\partial_t \log \pi_t = Q_t - Z_t$, which can be further written as $\partial_t \pi_t = \pi_t \cdot (Q_t - Z_t)$. Here we have $Q_t(a,s)=Q(a,s;\rho_t)$ and $Z_t$ is the continuous-time limit of $\hat Z_k$.  Furthermore, noting that $ \partial_t \int \pi_t(\rd a\given s) = 0 $, the mean-field limit of the DF-PPO update in \eqref{eq:ppo-discrete} is
\begin{align}
	\label{eq:ppo-cont}
	\text{MF-PPO:}\quad\frac{\rd}{\rd t} \pi_t = \pi_t \cdot A_t, \qquad \text{where}~ A_t(s, a) = Q_t(s, a) - \int Q_t(s, a) \pi_t(\rd a\given s).
\end{align}
The two updates \eqref{eq:td-cont} and \eqref{eq:ppo-cont} correspond to the mean-field limits of \eqref{eq:td-discrete} and \eqref{eq:ppo-discrete}, respectively, and together serve as the mean-field limit of the two-timescale AC. 
In particular, we remark that the MF-TD update in \eqref{eq:td-cont} for the critic is captured by a semigradient flow in the Wasserstein space \citep{villani2008optimal} while the MF-PPO update in \eqref{eq:ppo-cont} for the actor resembles the replicator dynamics \citep{schuster1983replicator, borgers1997learning, hennes2020neural}. See \S\ref{append:replicator dynamics} for more discussion of the replicator dynamics.
Note that such a framework is applicable to continuous state and action spaces. In this paper, we aim to provide a theoretical analysis of the mean-field limit of the two-timescale AC.

%% file: result.tex
\section{Main Result}\label{sec:main result}
In this section, we first establish the convergence of the MF-PPO update in \S\ref{sec:convergence of PPO}. Then, under additional assumptions, we establish the optimality and convergence of the mean-field two-timescale AC in \S\ref{sec:optimality of AC}.
\subsection{Convergence of Mean-field PPO}\label{sec:convergence of PPO}
For the MF-PPO update in \eqref{eq:ppo-cont}, we establish the following theorem on global optimality and convergence rate.
\begin{theorem}[Convergence of MF-PPO] \label{thm:ppo}
	Let $\pi^* = \argmax_{\pi} J(\pi)$ be the optimal policy and $\pi_0$ be the initial policy. 
	Then, it holds that
	\begin{align}\label{eq:conv-ppo}
		\frac{1}{T}\int_0^T \bigl( J(\pi^*) - J(\pi_t) \bigr) \rd t  \le \frac \zeta T  + 4\kappa\cdot \underbrace{\frac 1T\int_0^T  \norm{ Q_t - Q^{\pi_t} }_{2, \tphi^{\pi_t}}\rd t }_{\displaystyle\text{policy evaluation error}},
	\end{align}
	where $\tphi^{\pi_t}\in \sP(\cS\times\cA)$ is an evaluation distribution for the policy evaluation error and $\zeta = \EE_{s\sim \cE_{\cD_0}^{\pi^*}}\Bigl[ \kl\bigl(\pi^*(\cdot \given s) \,\|\, \pi_0(\cdot \given s) \bigr) \Bigr]$ is the expected KL-divergence between $\pi^*$ and $\pi_0$.
	Furthermore, letting 
	$\tphi^{\pi_t}=\frac{1}{2} \tphi_0+ \frac{1}{2} \phi_0 \otimes \pi_t$, where $\tphi_0\in \sP(\cS\times\cA)$ is a base distribution and $\phi_0=\int_\cA \tphi_0(\cdot,\rd a)$, 
	the concentrability coefficient $\kappa$ is given by 
	\begin{align*}
		\kappa= \norm[\bigg]{ \frac{\tcE_{\cD_0}^{\pi^*}}{\tphi_0}}_{\infty}.
	\end{align*}
\end{theorem}
\begin{proof}
	See \S\ref{sec:pf-thm-ppo} for a detailed proof.
\end{proof}
The concentrability coefficient commonly appears in the reinforcement learning literature 
\citep{szepesvari2005finite,munos2008finite,antos2008learning,farahmand2010error,scherrer2015approximate,farahmand2016regularized,lazaric2010analysis,liu2019neural,wang2019neural}.
In contrast to a more standard concentrability coefficient form, note that $\kappa$ is irrelevant to the update of the algorithm.
To show the convergence of the MF-PPO, our condition here is much weaker since we only need to specify a base distribution $\tphi_0$ such that $\kappa<\infty$.

Theorem \ref{thm:ppo} shows that the MF-PPO converges to the globally optimal policy at a rate of $\cO(T^{-1})$ up to the policy evaluation error. Such a theorem implies the global optimality and convergence of a double-loop AC algorithm, where the critic $Q_t$ is solved to high precision and the policy evaluation error is sufficiently small. In the sequel, we consider a more challenging setting, where the critic $Q_t$ is updated simultaneously along with the update of the actor's policy $\pi_t$. 

\subsection{Global Optimality and Convergence of Two-timescale AC}\label{sec:optimality of AC}
In this section, we aim to provide an upper bound on the policy evaluation error when the critic and the actor are updated simultaneously. Specifically, the actor is updated via MF-PPO in \eqref{eq:ppo-cont} and the critic $Q_t = Q(\cdot; \rho_t)$ is updated via the MF-TD in \eqref{eq:td-cont}.
The smooth function $\sigma$ in the parameterization of the Q function in \eqref{eq:nn-inf} is taken to be the following two-layer neural network,
\begin{align}
	\label{eq:nn}
	\sigma(s, a; \theta) = B_\beta  \cdot \beta(b) \cdot \tilde \sigma \bigl(w^\top(s, a, 1)\bigr),
\end{align}
where $\tilde \sigma: \RR\rightarrow \RR$ is the activation function, $\theta = (b, w)$ is the parameter, and $\beta: \RR\rightarrow (-1, 1)$ is an odd and invertible function with scaling hyper-parameter $B_\beta > 0$. 
It then holds that $D = d+2$, where $d$ and $D$ are the dimensions of $(s, a)$ and $\theta$, respectively. 
It is worth noting that the function class of $\int \sigma(s, a; \theta)  \rho(\rd \theta)$ for $\rho\in \sP_2(\RR^D)$ is the same as
\begin{align}
	\label{eq:func-class}
	\cF = \Bigl\{ \int \beta' \cdot \tilde \sigma\bigl(w^\top (s, a, 1) \bigr)  \nu(\rd \beta', \rd w) \Biggiven \nu \in \sP_2\bigl((-B_\beta, B_\beta) \times\RR^{d+1} \bigr) \Bigr\},
\end{align}
which captures a vast function class because of the universal function approximation theorem \citep{barron1993universal,pinkus1999approximation}.
We remark that we introduce the rescaling function $\beta$ in \eqref{eq:nn} to avoid the study of the space of probability measures over $(-B_\beta, B_\beta) \times\RR^{d+1} $ in \eqref{eq:func-class}, which has a boundary and thus lacks the regularity in the study of optimal transport.
Furthermore, note that we introduce a hyper-parameter $\alpha > 1$ in the Q function in \eqref{eq:nn-inf}. Thus, we are using $\alpha \cdot \cF$ to represent $\cF$, which causes an ``over-representation'' when $\alpha > 1$. 
Such over-representation appears to be essential for our analysis.  To anticipate briefly, we note that  $\alpha$ controls the gap in the average total reward over time when the relative time-scale $\eta$ is properly selected according to Theorem \ref{thm:ttac}. Furthermore, such an influence is imposed through Lemma \ref{lem:RAVF}, which shows that the Wasserstein distance between $\rho_0$ and $\rho_{\pi_t}$ is upper bounded by $O(1/\alpha)$. 
In what follows, we consider the initialization of the TD update to be $\rho_0 = \cN(0, I_D)$, which implies that $Q(s, a; \rho_0) = 0$.
We next impose a regularity assumption on the two-layer neural network $\sigma$.

\begin{assumption}[Regularity of the Neural Network]
	\label{asp:nn}
	For the two-layer neural network $\sigma$ defined in \eqref{eq:nn}, we assume that the following properties hold.
	\begin{itemize}
		\item[(i)] The rescaling function $\beta: \RR\rightarrow (-1, 1)$ is odd, $L_{0, \beta}$-Lipschitz continuous, $L_{1, \beta}$-smooth, and invertible. 
		Meanwhile, the inverse $\beta^{-1}$ is locally Lipschitz continuous. In particular, we assume that $\beta^{-1}$ is $\ell_\beta$-Lipschitz continuous in $ [-2/3, 2/3] $.
		\item[(ii)] The activation function $\tilde \sigma: \RR\rightarrow \RR$ is odd, $B_{\tilde \sigma}$-bounded, $L_{0, \tilde \sigma}$-Lipschitz continuous, and $L_{1, \tilde \sigma}$-smooth.
	\end{itemize}
\end{assumption}
We remark that Assumption \ref{asp:nn} is not restrictive and is satisfied by a large family of neural networks, e.g., $\tilde \sigma(x) = \tanh(x)$ and $\beta(b) = \tanh(b)$. Noting that $\norm{(s, a)}_2 \le 1$, Assumption \ref{asp:nn} implies that the function $\sigma(s, a; \theta)$ in \eqref{eq:nn} is odd  with respect to $w$ and $b$ and is also bounded, Lipschitz continuous, and smooth in the parameter domain, that is,
\begin{align}\label{eq:B_sigma}
	|\nabla_\theta\sigma(s,a;\theta)|<B_1,\quad
	|\nabla^2_{\theta\theta} \sigma(s,a;\theta)|<B_2.
\end{align}
We then impose the following assumption on the MDP.

\begin{assumption}[Regularity of the MDP] 
	\label{asp:mdp}
	For the MDP $ (\cS, \cA, \gamma, P, r, \cD_0) $, we assume the following properties hold.
	\begin{itemize}
		\itemsep 0em
		\item[(i)] The reward function $r$ and the transition kernel $P$ admit the following representations with respect to the activation function $\tilde\sigma$,
		\begin{align} \label{eq:mdp-r}
			r(s, a) &= B_r \cdot \int \tilde \sigma\bigl((s, a, 1)^\top w \bigr) \mu( \rd w), \\
			\label{eq:mdp-p}
			P(s'\given s, a) &= \int \tilde \sigma \bigl((s, a, 1)^\top w \bigr)  \varphi(s') \psi(s';\rd w),
		\end{align}
		where $\mu$ and $\psi(s'; \cdot)$ are probability measures in $\sP_2(\RR^{d+1})$ for any $s'\in \cS$, $B_r$ is a positive scaling parameter, and $\varphi(s'): \cS\rightarrow \RR_+$ is a nonnegative function.
		\item[(ii)] The reward function $r$ satisfies that $r(s, a) \ge 0$ for any $ (s, a) \in \cS\times \cA $. For the representation of $r$ in \eqref{eq:mdp-r} and the representation of the transition kernel $P$ in \eqref{eq:mdp-p}, we assume that 
		\begin{align*}
			&\chi^2(\mu \,\|\,\rho_{w,0}) < M_\mu, \qquad
			\chi^2\bigl(\psi(s;\cdot) \,\|\, \rho_{w,0}\bigr) < M_\psi,\quad \forall s\in \cS,\\
			&\int \varphi(s) \rd s \le M_{1, \varphi},\qquad \int  \varphi(s)^2 \rd s\le M_{2,\varphi},
		\end{align*} 
		where $\rho_{w,0}$ is the marginal distribution of $\rho_0$ with respect to $w$, i.e., $\rho_{w,0}=\int \rho_0(\rd b, \cdot)$, $\chi^2$ is the chi-squared divergence, and $M_\mu$, $M_\psi$, $M_{1, \varphi}$, $M_{2,\varphi}$ are absolute constants.
		%
		
		\item[(iii)]
		We assume that there exists an absolute constant $\cG$ such that
		\begin{gather*}
			\Hnorm[\big]{\psi(s;\cdot)-\psi(s';\cdot)}{\mu} <\cG,\qquad \Hnorm[\big]{\psi(s;\cdot)-\mu}{\mu}<\cG,\nonumber \\
			\Hnorm[\big]{\psi(s;\cdot)-\mu}{\psi(s';\cdot)}<\cG,
			\quad \Hnorm[\big]{\psi(s;\cdot)-\psi(s';\cdot)}{\psi(s'';\cdot)} <\cG ,\quad \forall s,s',s''\in \cS,
		\end{gather*}
		where $\Hnorm{\cdot}{\cdot}$ is the weighted homogeneous Sobolev norm.
	\end{itemize}
\end{assumption}
We remark that by assuming $\psi$ to be a probability measure and that $\varphi(s')\ge 0$ in \eqref{eq:mdp-p}, the representation of the transition kernel does not lose generality.
Specifically, the function class of \eqref{eq:mdp-p} is the same as 
\begin{align*}
	\cP= \Big\{ \int \tilde \sigma((s,a,1)^{\top} w) \tilde\psi(s';\rd w) \,\Big|\, \tilde\psi(s';\cdot) \text{ is a signed measure for any $s'\in \cS$}\Big\}.
\end{align*}
See \S\ref{lem:generality} for a detailed proof.
Assumption \ref{asp:mdp} generalizes the linear MDP in \cite{yang2019sample,yang2019reinforcement,cai2019provably} and \cite{jin2020provably}. In contrast, our representation of the reward function and the transition kernel benefits from the universal function approximation theorem and is thus not as restrictive as the original linear MDP assumption. 
Note that the infinite-width neural network has a two-layer structure by $\eqref{eq:nn}$. We establish the following lemma on the regularity of the representation of the action value function $Q^\pi$ by such a neural network.

\begin{lemma}[Regularity of Representation of $Q^\pi$] \label{lem:RAVF}
	Suppose that Assumptions \ref{asp:nn} and \ref{asp:mdp} hold.
	For any policy $\pi$, there exists a probability measure $\rho_{\pi} \in \sP_2(\RR^D)$ for the representation of $Q^\pi$ with the following properties.
	\begin{itemize}
		\item[(i)] For function $Q(s,a;\rho_\pi)$ defined by \eqref{eq:nn-inf} with $\rho=\rho_\pi$ and the action value function $Q^\pi(s,a)$ defined by \eqref{eq:def-q}, we have $Q(s,a; \rho_\pi) = Q^\pi(s,a)$ for any $(s,a)\in \cS\times \cA$.
		\item[(ii)] By letting  $ B_\beta \ge 2 (B_r + \gamma(1-\gamma)^{-1}B_rM_{1, \varphi})$ for the neural network defined in \eqref{eq:nn} and $\rho_0\sim \cN(0,I_D)$ for the initial distribution, we have $\tW_2(\rho_\pi, \rho_0)\le \bar{D}$ for any policy $\pi$,  where we define $\tW_2(\cdot,\cdot)=\alpha W_2(\cdot,\cdot)$ as the scaled $W_2$ metric. Here the constant $\bar{D}$ depends on the discount factor $\gamma$ and the absolute constants $ L_{0,\beta}, L_{1,\beta}, l_\beta, B_r,M_\mu,  M_\psi, M_{1,\varphi}, M_{2,\varphi}$ defined in Assumptions \ref{asp:nn} and \ref{asp:mdp}.
	\end{itemize}
\end{lemma}
\begin{proof}
	See \S\ref{sec:full-lem-RAVF} for a detailed proof.
\end{proof}
Property (i) of Lemma \ref{lem:RAVF} shows that the action value function $Q^\pi$ can be parameterized with the infinite-width two-layer neural network $Q(\cdot;\rho_\pi)$ in \eqref{eq:nn-inf}. 
Note that a larger $B_\beta$ captures a larger function class in \eqref{eq:func-class}.
Without loss of generality, we assume that $ B_\beta \ge 2 (B_r + \gamma(1-\gamma)^{-1}B_rM_{1, \varphi})$ holds in the sequel. Hence, by Property (ii), we have that $\tW_2(\rho_\pi, \rho_0)\le O(1)$ for any policy $\pi$.
In particular, it holds by Property (i) of Lemma \ref{lem:RAVF} that $\norm{ Q_t - Q^{\pi_t} }_{2, \tphi^{\pi_t}} = \norm{Q(\cdot;\rho_t)-Q(\cdot;\rho_{\pi_t})}_{2, \tphi^{\pi_t}}$ and we have the following theorem to characterize such an error with regard to the $W_2$ space.

\begin{theorem}[Upper Bound of Policy Evaluation Error] \label{thm:td}
	Suppose that Assumptions \ref{asp:nn} and \ref{asp:mdp} hold and $\rho_0\sim \cN(0, I_D)$ is the initial distribution.
	We specify the weighting distribution $\tPhi^{\pi_t}$ in MF-TD \eqref{eq:td-cont} as $\tPhi^{\pi_t}=\tcE_{\tphi^{\pi_t}}^{\pi_t}$, 
	where $\tphi^{\pi_t}\in\sP(\cS\times\cA)$ is the evaluation distribution for the policy evaluation error in Theorem \ref{thm:ppo}. Then, it holds that
	\begin{align}\label{eq:td}
		(1-\sqrt\gamma) \cdot \norm{Q_t - Q^{\pi_t}}_{2, \tphi^{\pi_t}}^2
		&\le -\frac{\rd}{\rd t} \frac{\tW_2^2(\rho_t,\rho_{\pi_t})}{2\eta}
		+ \Delta_t,
	\end{align}
	where
	\begin{align*}
		\Delta_t=
		&2 \alpha^{1/2}\eta^{-1}\cB  B_1 \cdot \tW_2(\rho_t, \rho_0) \tW_2(\rho_t, \rho_{\pi_t})\nonumber\\
		&\quad+ \alpha^{-1} B_2 \cdot \Big(4 B_1 \max\big\{\tW_2(\rho_{\pi_t}, \rho_0), \tW_2(\rho_t,\rho_0)\big\} + B_r\Big) \tW_2(\rho_t, \rho_\pi)^2.
	\end{align*}
	Here $B_1$ and $B_2$ are defined in \eqref{eq:B_sigma} of Assumption \ref{asp:nn}, 
	$\eta$ is the relative TD timescale, $\alpha$ is the scaling parameter of the neural network, and $\tW_2=\alpha W_2$ is the scaled $W_2$ metric. Moreover, the constant $\cB$  depends on the discount factor $\gamma$, the scaling parameter $B_\beta$ in \eqref{eq:nn}, and the absolute constants $ l_\beta, B_r, M_{1,\varphi}, \cG$ defined in Assumptions \ref{asp:nn} and \ref{asp:mdp}.
	\begin{proof}
		See \S\ref{sec:pf-thm-td} for a detailed proof.
	\end{proof}
\end{theorem}
Here we give a nonrigorous discussion on how to upper bound $\Delta_t$ in \eqref{eq:td}.
If $\tW_2(\rho_t, \rho_0) \le O(1)$ holds for any $t\in [0,T]$, by $\tW_2(\rho_{\pi_t}, \rho_0)\le O(1)$ in Lemma \ref{lem:RAVF} and the triangle inequality of $W_2$ distance \citep{villani2008optimal}, it follows that $\tW_2(\rho_t,\rho_{\pi_t})\le O(1)$ and $\Delta_t \le O(\alpha^{1/2}\eta^{-1}+\alpha^{-1})$.
Taking a time average of integration on both sides of \eqref{eq:td}, the policy evaluation error $\PPE$ is then upper bounded by $O(\eta^{-1} T^{-1}+\alpha^{1/2}\eta^{-1}+\alpha^{-1})$.
Inspired by such a fact, we introduce the following restarting mechanism to ensure $\tW_2(\rho_t,\rho_0)\le O(1)$.
\Vskip
\noindent{\bf Restarting Mechanism.}
Let $\tW_0=\lambda \bar D$ be a threshold, where $\bar D$ is the upper bound for $\tW_2(\rho_\pi,\rho_0)$ by Lemma \ref{lem:RAVF}, $\lambda\ge 3$ is a constant scaling parameter for the restarting threshold,
$\rho_t$ is the distribution of the parameters in the neural network at time $t$, and $\rho_0$ is the initial distribution.
Whenever we detect that $\tW_2(\rho_t,\rho_0)$ reaches $\tW_0$ in the update, we pause and reset $\rho_t$ to $\rho_0$ by resampling the parameters from $\rho_0$.
Then, we reset the critic with the newly sampled parameters while keeping the actor's policy $\pi_t$ unchanged and continue the update.

The restarting mechanism guarantees $\tW_2(\rho_t, \rho_0)\le \lambda \bar D$ by restricting the distribution $\rho_t$ of the parameters to be close to $\rho_0$. 
Moreover, by letting $\lambda\ge 3$, we ensure that $\rho_{\pi_t}$ is realizable by $\rho_t$ since $\tW_2(\rho_{\pi_t}, \rho_0)\le \bar D\le \lambda \bar D$, which means that the neural network is capable of capturing the representation of the action value function $Q^{\pi_t}$.
We remark that by letting $\tW_0= O(1)$, we allow $\rho_t$ to deviate from $\rho_0$ up to $W_2(\rho_t,\rho_0)\le O(\alpha^{-1})$ in the restarting mechanism.
In contrast, the NTK regime \citep{cai2019neural} which corresponds to letting $\alpha=\sqrt M$ in \eqref{eq:Q-para} only allows $\rho_t$ to deviate from $\rho_0$ by the chi-squared divergence $\chi^2(\rho_t\d|\rho_0)\le O(M^{-1})=o(1)$.
That is, the NTK regime fails to induce a feature representation significantly different from the initial one.
Before moving on, we summarize the construction of the weighting distribution $\tPhi^{\pi_t}$ in Theorem \ref{thm:ppo} and \ref{thm:td} as follows,
\begin{align}\label{eq:tPhi-construct}
	\tPhi^{\pi_t}=\tcE_{\tphi^{\pi_t}}^{\pi_t}, \qquad\tphi^{\pi_t}=\frac{1}{2} \tphi_0+\frac{1}{2} \phi_0\otimes \pi_t,\qquad \phi_0=\int_\cA \tphi_0(\cdot,\rd a),
\end{align}
where $\tphi_0$ is the base distribution. 
Now we have the following theorem that characterizes the global optimality and convergence of the two-timescale AC with restarting mechanism.

\begin{theorem}[Global Optimality and Convergence Rate of Two-timescale AC with Restarting Mechanism]
	\label{thm:ttac}
	Suppose that \eqref{eq:tPhi-construct} and Assumptions \ref{asp:nn} and \ref{asp:mdp} hold. 
	With the restarting mechanism, it holds that
	\begin{align}\label{eq:ttac-error}
		&\frac 1T\int_0^T \bigl( J(\pi^*) - J(\pi_t) \bigr) \rd t 
		\le \underbrace{\frac \zeta T }_{\displaystyle\text{(a)}}
		+\underbrace{4\kappa \sqrt{\alpha^{-1} S_1 +\alpha^{1/2} \eta^{-1} S_2 
				+\frac{\eta^{-1} \bar D^2 }{2  T(1-\sqrt\gamma)}}}_{\displaystyle\text{(b)}},
	\end{align}
	where we have
	\begin{gather*}
		\zeta = \EE_{s\sim \cE_{\cD_0}^{\pi^*}}\Bigl[ \kl\bigl(\pi^*(\cdot \given s) \,\|\, \pi_0(\cdot \given s) \bigr) \Bigr], \quad\kappa = \norm[\bigg]{ \frac{\tcE_{\cD_0}^{\pi^*}}{\tphi_0}}_{\infty},\\
		S_1=\frac{(1+\lambda)^2 \bar D^2 B_2(4B_1 \lambda \bar D+B_r)}{1-\sqrt\gamma}, \quad S_2=\frac{2\cB B_1 \lambda(1+\lambda) \bar D^2}{1-\sqrt\gamma}.
	\end{gather*}
	Here $B_r$, $B_1$ and $B_2$ are defined	in Assumption \ref{asp:nn} and \ref{asp:mdp}, $\bar D$ is the upper bound for $\tW_2(\rho_\pi,\rho_0)$ in Lemma \ref{lem:RAVF}, $\cB$ depends on the discount factor $\gamma$ and the absolute constants defined in Assumption \ref{asp:nn} and \ref{asp:mdp}, and $\lambda$ is the scaling parameter for the restarting threshold. Additionally, the total restarting number $N$ satisfies the following inequality,
	\begin{align*}
		N\le (\lambda-2)^{-1} \big((\alpha^{-1} \eta S_1  +\alpha^{1/2} S_2)2T\bar D^{-2}(1-\sqrt\gamma)+1\big).
	\end{align*}	
	\begin{proof}
		See \S\ref{sec:pf-thm-ttac} for a detailed proof.
	\end{proof}
\end{theorem}

Note that for a given MDP with starting distribution $\cD_0$, the expected KL-divergence $\zeta$ and the concentrability coefficient $\kappa$ are both independent of the two-timescale update.
We remark that our condition for \eqref{eq:ttac-error} to be bounded is not restrictive.
Specifically, we only need a given $\pi_0$ and $\tphi_0$ such that the KL-divergence $\zeta<\infty $ and the concentrability coefficient $\kappa<\infty$, which is a weakening relative to the standard usage of the concentrability coefficient in the literature \citep{szepesvari2005finite,munos2008finite,antos2008learning,farahmand2010error,scherrer2015approximate,farahmand2016regularized,lazaric2010analysis,liu2019neural,wang2019neural}.

The first term (a) on the right-hand side of \eqref{eq:ttac-error} diminishes as $T\rightarrow \infty$.
The second term (b) corresponds to the policy evaluation error. 
We give an example to demonstrate the convergence of the two-timescale AC. We let the scaling parameter $\lambda=3$ for the restarting threshold. By letting $\eta=\alpha^{3/2}$, it holds that $\frac 1T\int_0^T \bigl( J(\pi^*) - J(\pi_t) \bigr) \rd t $ decreases at a rate of $ O(T^{-1}+O(\alpha^{-1/2})+O(\alpha^{-3/4}T^{-1/2}))$ as $\alpha \rightarrow \infty$ and $T\rightarrow \infty$.
Note that $\eta= \alpha^{3/2}$ shows that the critic has a larger relative TD timescale in \eqref{eq:td-cont}.
As for the total number of restartings $N$, it holds that $N \le O(\alpha^{1/2}T)$ as $\alpha\rightarrow \infty$, which induces a tradeoff, i.e., a larger $\alpha$ guarantees a smaller gap in $\frac 1T\int_0^T \bigl( J(\pi^*) - J(\pi_t) \bigr) \rd t$  but yields more restartings and potentially requires a larger relative TD timescale.

%% file: W_2.tex
\section{Supplement to the Background}
In this section, we present some background on Wasserstein space and replicator dynamics. 
\subsection{Wasserstein Space}\label{append:wasserstein space}
Let $\Theta \subseteq \RR^D$ be a Polish space. We denote by $\sP_2(\Theta) \subseteq \sP(\Theta)$ the set of probability measures with finite second moments. Then, the Wasserstein-2 ($W_2$) distance between $\mu, \nu \in \sP_2(\Theta)$ is defined as follows,
\begin{align}
\label{eq:w2-def-1}
W_2(\mu, \nu) = \inf\Bigl\{ \EE\bigl[\norm{X - Y}^2\bigr]^{1/2} \Biggiven \law (X) = \mu, \law (Y) = \nu  \Bigr\},
\end{align} 
where the infimum is taken over the random variables $X$ and $Y$ on $\Theta$ and we denote by ${\rm law}(X)$ the distribution of a random variable $X$.
We call $\cM = (\sP_2(\Theta), W_2)$ the Wasserstein space, which is an infinite-dimensional manifold \citep{villani2008optimal}. 
In particular,  we define the tangent vector at $\mu \in \cM$ as $\dot\rho_0$ for the corresponding curve $\rho: [0, 1] \rightarrow \sP_2(\Theta)$ with $\rho_0 = \mu$. 
Under certain regularity conditions, the continuity equation $\partial_t \rho_t = - \Div(\rho_t v_t)$ corresponds to a vector field $v : [0,1 ]\times \Theta \rightarrow \RR^D$, which endows the infinite-dimensional manifold $\sP_2(\Theta)$ with a weak Riemannian structure in the following sense \citep{villani2008optimal}. Given any tangent vectors $u$ and $\tilde u$ at $\mu\in \cM$ and the corresponding vector fields $v, \tilde v$, which satisfy $u + \Div (\mu v) = 0$ and $\tilde u + \Div(\mu \tilde v) = 0$, respectively, we define the inner product of $u$ and $\tilde u$ as follows,
\begin{align}
\label{eq:w-inner}
\inp{u}{\tilde u}_{\mu, W_2} = \int v \cdot \tilde v \rd \mu = \inp{v}{\tilde v}_{\mu},
\end{align}
which yields a Riemannian metric. Such a Riemannian metric further induces a norm $\norm{u}_{\mu,W_2} = \inp{u}{u}_{\mu,W_2}^{1/2}$ for any tangent vector $u\in T_\mu \cM$ at any $\mu\in \cM$, which allows us to write the Wasserstein-2 distance defined in \eqref{eq:w2-def-1} as follows,
\begin{align}
\label{eq:w2-def-2}
W_2(\mu, \nu) = \inf\Biggl\{ \biggl( \int_0^1 \norm{\dot\rho_t}^2_{\rho_t,W_2} \,\rd t\biggr)^{1/2} \Bigggiven \rho:[0, 1] \rightarrow \cM, \rho_0 = \mu, \rho_1 = \nu\Biggr\}.
\end{align}
Here $\dot\rho_s$ denotes $\partial_t \rho_t\given_{t=s}$ for any $s\in [0,1]$. In particular, the infimum in \eqref{eq:w2-def-2} is attained by the geodesic $\tilde \rho: [0,1] \rightarrow \sP_2(\Theta)$ connecting $\mu, \nu \in \cM$. Moreover, the geodesics on $\cM$ are constant-speed, that is, 
\begin{align}
\label{eq:constant-speed}
\norm{\dot{\tilde \rho_t}}_{\tilde \rho_t,W_2} = W_2(\mu, \nu),\quad \forall t\in[0,1].
\end{align}
In Wasserstein space $\cM$, a curve $\rho:[0,1]\rightarrow \sP_2(\Theta)$ is defined to be absolutely continous if there exists $m\in L^1(a,b)$, i.e., $\int_a^b \normf{\dot m(t)} \rd t<\infty$, such that
\begin{align*}
	W_2(\rho_s,\rho_t)\le \int_s^t m(r) \rd r,\quad \forall a<s\le t<b.
\end{align*}
Such an absolutely continuous curve $\rho_t$ allows us to define the metric derivative in $\cM$ as follows,
\begin{align}\label{eq:metric derivative}
	\normf{\dot\rho_t}_{W_2}=\lim_{s\rightarrow t}\frac{W_2(\rho_s,\rho_t)}{\normf{s-t}}.
\end{align}
By \cite{ambrosio2008gradient}, the metric derivative $\normf{\dot \rho_t}_{W_2}$ is connected to the norm of the tangent vector by
\begin{align}\label{eq:MN=VN}
	\normf{\dot \rho_t}_{W_2}=\norm{\dot \rho_t}_{\rho_t, W_2}.
\end{align}
Furthermore, we introduce the Wasserstein-1 distance, which is defined as 
	\begin{align*}
	W_1(\mu^1, \mu^2) = \inf\Bigl\{ \EE\bigl[\norm{X-Y}\bigr] \Biggiven {\rm law}(X) = \mu^1, {\rm law}(Y) = \mu^2 \Bigr\}
	\end{align*}
	for any $\mu^1, \mu^2 \in\sP(\RR^D)$ with finite first moments.
	The Wasserstein-1 distance has the following dual representation \citep{ambrosio2008gradient},
	\begin{align}
	\label{eq:w1-dual}
	W_1(\mu^1, \mu^2) = \sup\biggl\{ \int f(x) \,\rd (\mu^1-\mu^2)(x) \bigggiven {\rm continuous}~f: \RR^D \rightarrow \RR, \lip(f) \le 1 \biggr\}.
	\end{align}

\subsection{Replicator Dynamics}\label{append:replicator dynamics}
The replicator dynamics originally arises in the study of evolutionary game theory \citep{schuster1983replicator}.
For a function $f$, the replicator dynamics is given by the differential equation \begin{align*}
    \frac{\rd}{\rd t}x_t(a) = x_t(a) [f(a, x_t) - \phi(x)],
\end{align*}
where $\phi(x) = \int x(a)f(a, x)$. As for the PPO update in \eqref{eq:ppo-cont}, for a fixed $s$, let $x(a) = \pi(a \,|\, s)$ and $f(a, x) = Q^\pi(s, a)$, we see that \eqref{eq:ppo-cont} corresponds to a replicator dynamics if $Q_t = Q^{\pi_t}$. Note that in the simultaneous update of both the critic and actor, we do not have access to the true action value function $Q^\pi$. Thus, we use the estimator $Q_t$ calculated by the critic step to guide the update of the actor in the PPO update, which takes the form of a replicator dynamics in the continuous-time limit.

%% file: theory.tex
\section{Proofs of Main Results}
In this section, we give detailed proof of the theorems and present a detailed statement of Lemma \ref{lem:RAVF}.
\subsection{Proof of Theorem \ref{thm:ppo}}\label{sec:pf-thm-ppo}
\begin{proof}
	Following from the performance difference lemma \citep{kakade2002approximately}, we have
	\begin{align*}
		J(\pi^*) - J(\pi_t) &= (1- \gamma)^{-1} \cdot \EE_{s\sim \cE^{\pi^*}_{\cD_0}} \bigl[ \inp{A^{\pi_t}(s, \cdot)}{\pi^*(\cdot \given s) - \pi_t(\cdot \given s)}_\cA \bigr],
	\end{align*}
	where $\cE^{\pi^*}_{\cD_0}$ is the visitation measure induced by $\pi^*$ from $\cD_0$ and $A^{\pi_t}$ is the advantage function. Note that the continuous PPO dynamics in \eqref{eq:ppo-cont} can be equivalently written as $\partial_t \log \pi_t = A_t$. Thus, we have
	\begin{align}
		(1-\gamma) \cdot \bigl( J(\pi^*) - J(\pi_t)\bigr) &= \EE_{s\sim \cE^{\pi^*}_{\cD_0}}\Bigl[ \inp[\big]{\frac{\rd}{\rd t} \log \pi_t(\cdot \given s)+A^{\pi_t}(s, \cdot ) - A_t(s, \cdot)}{\pi^*(\cdot \given s) - \pi_t(\cdot \given s)}_\cA\Bigr] \nonumber \\
		&= \EE_{s\sim \cE^{\pi^*}_{\cD_0}} \Bigl[ \inp[\big]{\frac{\rd}{\rd t} \log \pi_t(\cdot \given s)}{\pi^*(\cdot \given s) - \pi_t(\cdot \given s)}_\cA \Bigr]  \nonumber\\
		&\quad + \EE_{s\sim \cE^{\pi^*}_{\cD_0}}\Bigl[ \inp[\big]{Q^{\pi_t}(s, \cdot ) - Q_t(s, \cdot)}{\pi^*(\cdot \given s) - \pi_t(\cdot \given s)}_\cA\Bigr]. \label{eq:ppo2}
	\end{align}
	For the first term on the right-hand side of \eqref{eq:ppo2}, it holds that,
	\begin{align}\label{eq:ppo21}
		&\inp[\big]{\frac{\rd}{\rd t} \log \pi_t(\cdot \given s)}{\pi^*(\cdot \given s) - \pi_t(\cdot \given s)}_\cA \nonumber \\ 
		&\quad= \inp[\big]{\frac{\rd}{\rd t} \log \pi_t(\cdot \given s)}{\pi^*(\cdot \given s)}_\cA - \inp[\big]{\frac{\rd}{\rd t} \log \pi_t(\cdot \given s)}{\pi_t(\cdot \given s)}_\cA \\ \nonumber
		&\quad =-\frac{\rd}{\rd t}\kl\bigl(\pi^*(\cdot \given s) \,\|\, \pi_t(\cdot \given s) \bigr),
	\end{align}
	where the last equality holds by noting that $\inp[\big]{\partial_t \log \pi_t(\cdot \given s)}{\pi_t(\cdot \given s)}_\cA = \partial_t\int_\cA  \pi_t(\rd a\given s)=0$.
	For the second term on the right-hand side of \eqref{eq:ppo2}, by the Cauchy-Schwartz inequality, we have
	\begin{align}
		\label{eq:ppo22}
		\EE_{s\sim \cE^{\pi^*}_{\cD_0}}\Bigl[ \inp[\big]{Q^{\pi_t}(s, \cdot ) - Q_t(s, \cdot)}{\pi^*(\cdot \given s)}_\cA\Bigr] 
		& = \EE_{(s, a)\sim \tphi^{\pi_t}} \Bigl[ \bigl( Q^{\pi_t}(s, a ) - Q_t(s, a)\bigr) \pi^*(a|s) \frac{\cE^{\pi^*}_{\cD_0}(s)}{\tphi^{\pi_t}(x)}\Bigr] \nonumber \\
		&  \le \norm[\bigg]{\frac{\tcE^{\pi^*}_{\cD_0}}{\tphi^{\pi_t}}}_{2, \tphi^{\pi_t}} \cdot \norm[\big]{Q^{\pi_t} - Q_t }_{2, \tphi^{\pi_t}}, \\
		\label{eq:ppo221}
		\EE_{s\sim \cE^{\pi^*}_{\cD_0}}\Bigl[ \inp[\big]{Q^{\pi_t}(s, \cdot ) - Q_t(s, \cdot)}{\pi_t(\cdot \given s)}_\cA\Bigr] 
		& = \EE_{(x) \sim \tphi^{\pi_t}}\Bigl[ \big(Q^{\pi_t}(s, \cdot ) - Q_t(s, \cdot)) \pi_t(a \given s) \frac{\cE^{\pi^*}_{\cD_0}(s)}{\tphi^{\pi_t}(x)}\Bigr] \nonumber\\
		&\le \norm[\bigg]{ \frac{\cE^{\pi^*}_{\cD_0}\otimes \pi_t}{\tphi^{\pi_t}}}_{2,\tphi^{\pi_t}} \cdot \norm[\big]{Q^{\pi_t} - Q_t }_{2, \tphi^{\pi_t}}.
	\end{align}
	Plugging \eqref{eq:ppo21}, \eqref{eq:ppo22}, and \eqref{eq:ppo221} into \eqref{eq:ppo2}, we have
	\begin{align} \label{eq:ppo-diff}
		J(\pi^*) - J(\pi_t)  &\le -\frac{\rd}{\rd t}\EE_{s\sim \cE^{\pi^*}_{\cD_0}}\Bigl[ \kl\bigl(\pi^*(\cdot \given s) \,\|\, \pi_t(\cdot \given s) \bigr) \Bigr]  \nonumber \\ &\quad+\Big(\norm[\bigg]{\frac{\tcE_{\cD_0}^{\pi^*}}{\tphi^{\pi_t}}}_{2,\tphi^{\pi_t}}+\norm[\bigg]{\frac{\cE^{\pi^*}_{\cD_0}\otimes \pi_t}{\tphi^{\pi_t}}}_{2, \tphi^{\pi_t}}\Big) \cdot \norm{ Q_t - Q^{\pi_t} }_{2, \tphi^{\pi_t}}.
	\end{align}
By further letting $\tilde \phi^{\pi_t}=\frac{1}{2} \tphi_0+ \frac{1}{2} \phi_0 \otimes \pi_t$, it holds for the concentrability coefficient $\norm[\bigg]{\frac{\tcE_{\cD_0}^{\pi^*}}{\tphi^{\pi_t}}}_{2,\tphi^{\pi_t}}$ in \eqref{eq:ppo22} that
\begin{align}\label{eq:c-coe1}
	\norm[\bigg]{\frac{\tcE_{\cD_0}^{\pi^*}}{\tphi^{\pi_t}}}_{2,\tphi^{\pi_t}}
	\le \norm[\bigg]{\frac{2\tcE_{\cD_0}^{\pi^*}}{ \tphi_0}}_{2,\tphi^{\pi_t}}
	\le 2\norm[\bigg]{\frac{\tcE_{\cD_0}^{\pi^*}}{\tphi_0}}_\infty.	
\end{align}
By further letting $\phi_0(s)=\int_\cA \tphi_0(s, \rd a)$, it holds for the concentrability coefficient $\norm[\bigg]{\frac{\cE^{\pi^*}_{\cD_0}\otimes \pi_t}{\tphi^{\pi_t}}}_{2, \tphi^{\pi_t}}$ in \eqref{eq:ppo221} that
\begin{align}\label{eq:c-coe2}
	\norm[\bigg]{\frac{\cE^{\pi^*}_{\cD_0}\otimes \pi_t}{\tphi^{\pi_t}}}_{2, \tphi^{\pi_t}}
	\le \norm[\bigg]{\frac{ 2\cE^{\pi^*}_{\cD_0}}{\phi_0}}_{2, \tphi^{\pi_t}}
	\le  2\norm[\Bigg]{\frac{\int \frac{\tcE_{\cD_0}^{\pi^*}(x)}{\tphi_0(x)}\tphi_0(s,\rd a) }{\phi_0(s)}}_\infty \le 2\norm[\bigg]{ \frac{\tcE_{\cD_0}^{\pi^*}}{\tphi_0}}_{\infty}.
\end{align}
Plugging \eqref{eq:c-coe1} and \eqref{eq:c-coe2} into \eqref{eq:ppo-diff} and taking integration on both sides of \eqref{eq:ppo-diff}, we have
\begin{align*} 
	\frac{1}{T}\int_0^T \bigl( J(\pi^*) - J(\pi_t) \bigr) \rd t  \le \frac 1T \EE_{s\sim \cE^{\pi^*}_{\cD_0}}\Bigl[ \kl\bigl(\pi^*(\cdot \given s) \,\|\, \pi_0(\cdot \given s) \bigr) \Bigr] + \frac{4\kappa}{T}\cdot\int_0^T  \norm{ Q_t - Q^{\pi_t} }_{2, \tphi^{\pi_t}}\rd t,
\end{align*}
where $\kappa=\norm[\big]{ {\tcE_{\cD_0}^{\pi^*}}/{\tphi_0}}_{\infty}$.
Thus, we complete the proof of Theorem \ref{thm:ppo}.
\end{proof}

\subsection{Detailed Statement of Lemma \ref{lem:RAVF}}\label{sec:full-lem-RAVF}
We give a detailed version of Lemma \ref{lem:RAVF} as follows.
\begin{lemma}[Regularity of Representation of $Q^\pi$] \label{lem:full-RAVF}
	Suppose that Assumptions \ref{asp:nn} and \ref{asp:mdp} hold.
	For any policy $\pi$, there exists a probability measure $\rho_{\pi} \in \sP_2(\RR^D)$ for the representation of $Q^\pi$ with the following properties.
	\begin{itemize}
		\item[(i)] For function $Q(s,a;\rho_\pi)$ defined by \eqref{eq:nn-inf} with $\rho=\rho_\pi$ and the action value function $Q^\pi(s,a)$ defined by \eqref{eq:def-q}, we have $Q(s,a; \rho_\pi) = Q^\pi(s,a)$ for any $(s,a)\in \cS\times \cA$.
		\item[(ii)] For $g$ defined in \eqref{eq:g-rho}, we have $g(\cdot; \rho_\pi) = 0$ for any policy $\pi$.
		\item[(iii)] By letting  $ B_\beta \ge 2 (B_r + \gamma(1-\gamma)^{-1}B_rM_{1, \varphi})$ for the neural network defined in \eqref{eq:nn} and $\rho_0\sim \cN(0,I_D)$ for the initial distribution, we have $\tW_2(\rho_\pi, \rho_0)<\bar{D}$ for any policy $\pi$,  where we define $\tW(\cdot,\cdot)=\alpha W(\cdot,\cdot)$ as the scaled $W_2$ metric. Here constant $\bar{D}$ depends on the discount factor $\gamma$ and the absolute constants $ L_{0,\beta}, L_{1,\beta}, l_\beta, B_r,M_\mu,  M_\psi, M_{1,\varphi}, M_{2,\varphi}$ defined in Assumptions \ref{asp:nn} and \ref{asp:mdp}.
		\item[(iv)] For any two policies $\pi_1$ and $\pi_2$, it holds that, $$W_2(\rho_{\pi_1},\rho_{\pi_2})\le \alpha^{-\frac{1}{2}} \cB\cdot \sup_{s\in \cS} \EE_{s' \sim \cE_{s}^{\pi_1}} \Big[\norm[\big]{\pi_1(\cdot\given s') - \pi_2(\cdot\given s')}_1\Big],$$ where $\cB$  depends on the dicount factor $\gamma$, the scaling parameter $B_\beta$ in \eqref{eq:nn}, and the absolute constants $ l_\beta, B_r, M_{1,\varphi}, \cG$ defined in Assumptions \ref{asp:nn} and \ref{asp:mdp}.
	\end{itemize}
\end{lemma}
\begin{proof}
	See \S\ref{sec:pf-lem-RAVF} for a detailed proof.
\end{proof}

\subsection{Proof of Theorem \ref{thm:td}}\label{sec:pf-thm-td}
\begin{proof}
    For notation simplicity, we let $x=(s,a)$.
	By Property (i) of Lemma \ref{lem:full-RAVF}, it holds that $\norm{Q_t - Q^{\pi_t}}_{2, \tphi^{\pi_t}}^2=\norm{Q_t - Q(\cdot;\rho_\pi)}_{2, \tphi^{\pi_t}}^2$, where $Q_t=Q(\cdot;\rho_t)$.
	Thus it prompts us to study the $W_2$ distance between $\rho_t$ and $\rho_{\pi_t}$.
	By the first variation formula in Lemma \ref{lem:diff}, it holds that
	\begin{align}\label{eq:w2-1}
		\frac{\rd}{\rd t} \frac{W_2^2(\rho_t,\rho_{\pi_t})}{2}=-\inp{\dot \rho_t}{\dot{\talpha_t^0}}_{\rho_t,W_2} - \inp{\dot \rho_{\pi_t}}{\dot{\tbeta_t^0}}_{\rho_{\pi_t},W_2}.
	\end{align}
	Here $\talpha_t^{[0,1]}$ is the geodesic connecting $\rho_t$ and $\rho_{\pi_t}$, and $\tbeta_t^{[0,1]}$ is its time-inverse, i.e., $\tbeta_t^s=\talpha_t^{1-s}$.
	Besides, we denote by $\dot {\talpha_t^s}=\partial_s \talpha_t^s$ and $\dot {\tbeta_t^s}=\partial_s \tbeta_t^s $ the derivation of geodesics $\talpha_t^s$ and $\tbeta_t^s$ with respect to $s$, respectively.
	We  denote by $v_s$  the corresponding vector field at $\talpha_t^s$, which  satisfies $ \partial_s \talpha_t^s = - \Div(\talpha_t^s v_s) $.
    For the first term of \eqref{eq:w2-1}, it holds that
	\begin{align}\label{eq:w2-1.1}
		-\inp{\dot \rho_t}{\dot {\talpha_t^0}}_{\rho_t,W_2} & = -\eta \cdot \inp[\big]{g(\cdot; \rho_t, \pi_t)}{v_0}_{\rho_t} \nonumber\\
		&= \eta \cdot \int_0^1\partial_s \inp[big]{g(\cdot; \talpha_t^s, \pi_t)}{v_s}_{\talpha_t^s} \rd s 
		- \eta \cdot \inp[\big]{g(\cdot; \rho_{\pi_t}, \pi_t)}{v_1}_{\rho_{\pi_t}} \nonumber\\
		& = \eta \cdot \int_0^1 \inp[\big]{\partial_s g(\cdot; \talpha_t^s, \pi_t)}{v_s}_{\talpha_t^s} \rd s 
		+ \eta \cdot \int_0^1 \int g(\theta; \talpha_t^s, \pi_t) \cdot \partial_s(v_s \talpha_t^s)(\theta) \rd \theta \rd s,
	\end{align}
	where the first equality follows from \eqref{eq:w-inner} and the third equation follows from $g(\cdot;\rho_{\pi_t},\pi_t)=0$ by Property (ii) of Lemma \ref{lem:full-RAVF}.
	For the first term on the right-hand side of \eqref{eq:w2-1.1}, we have
	\begin{align}\label{eq:w2-1.1-1}
		&\eta \cdot \int_0^1 \inp[\big]{\partial_s g(\cdot; \talpha_t^s, \pi_t)}{v_s}_{\talpha_t^s} \rd s \nonumber\\
		& \quad=-\alpha^{-1} \eta \cdot\int_0^1 \int \inp[\bigg]{ \EE_{\tPhi^{\pi_t}}^{\pi_t}\Bigl[ \partial_s\bigl( Q(x; \talpha_t^s) - \gamma \cdot Q(x'; \talpha_t^s) \bigr) \nabla \sigma(x; \theta) \Bigr]}{ (v_s \talpha_t^s)(\theta)} \rd \theta \rd s \nonumber\\
		&\quad= \alpha^{-1} \eta \cdot\int_0^1 \int \inp[\bigg]{ \EE_{\tPhi^{\pi_t}}^{\pi_t}\Bigl[ \partial_s\bigl( Q(x; \talpha_t^s) - \gamma \cdot Q(x'; \talpha_t^s) \bigr) \sigma(x; \theta) \Bigr]}{ \Div (v_s \talpha_t^s)(\theta)} \rd \theta \rd s,
	\end{align}
	where the first equality holds by defintion of $g$ in \eqref{eq:g-rho} and the last equality follows from Stokes' formula.
	Note that we have $\Div(v_s\talpha_t^s)=-\partial_s \talpha_t^s$ by the definition of vector field $v_s$. Thus, it holds for \eqref{eq:w2-1.1-1} that
	\begin{align}\label{eq:w2-1.1-2}
		&\eta \cdot \int_0^1 \inp[\big]{\partial_s g(\cdot; \talpha_t^s, \pi_t)}{v_s}_{\talpha_t^s} \rd s\nonumber \\
		&\quad= -\alpha^{-1} \eta \cdot\int_0^1 \int \inp[\bigg]{ \EE_{\tPhi^{\pi_t}}^{\pi_t}\Bigl[ \partial_s \bigl(Q(x; \talpha_t^s) - \gamma \cdot Q(x'; \talpha_t^s) \bigr) \sigma(x; \theta) \Bigr]}{ \partial_s \talpha_t^s(\theta) } \rd \theta \rd s \nonumber\\
		&\quad = - \alpha^{-2}\eta\cdot\int_0^1 \EE_{\tPhi^{\pi_t}}^{\pi_t}\Bigl[ \partial_s \bigl(Q(x; \talpha_t^s) - \gamma \cdot Q(x'; \talpha_t^s) \bigr) \partial_s Q(x; \talpha_t^s)\Bigr] \rd s,
	\end{align}
	where the last equality follows from the definition of $Q$ in \eqref{eq:nn-inf}.
	We let $f(\tilde D)=\big(\EE_{x\sim\tcE_\tcD^{\pi_t}}\bigl[ (\partial_s Q(x; \talpha_t^s))^2 \bigr]\big)^{1/2}$ with respect to a specific $s$ and $t$. Recall that for the weighting distribution $\tPhi^{\pi_t}$, we set $\tPhi^{\pi_t}=\tcE_{\tphi^{\pi_t}}^{\pi_t}$. Hence, for the integrand of \eqref{eq:w2-1.1-2}, we have
	\begin{align}\label{eq:w2-1.1-3}
		&-\EE_{\tPhi^{\pi_t}}^{\pi_t} \Bigl[ \partial_s \bigl(Q(x; \talpha_t^s) - \gamma \cdot Q(x'; \talpha_t^s) \bigr) \partial_s Q(x; \talpha_t^s)\Bigr] \nonumber\\
		&\quad =-f(\tphi^{\pi_t})^2+ \gamma\cdot\int \partial_s Q(x; \talpha_t^s) \partial_s Q(x'; \talpha_t^s) \tP^{\pi_t}(x'\given x) \tcE_{\tphi^{\pi_t}}^{\pi_t}(x)\rd x' \rd x\nonumber\\
		&\quad \le -f(\tphi^{\pi_t})^2 +\gamma \cdot\sqrt{\int \big(\partial_s Q(x;\talpha_t^s)\big)^2 \tcE_{\tphi^{\pi_t}}^{\pi_t}(\rd x)} \cdot \sqrt{\int \big(\partial_s Q(x';\talpha_t^s)\big)^2 \tP^{\pi_t}(\rd x'\given x) \tcE_{\tphi^{\pi_t}}^{\pi_t}(\rd x)},
	\end{align}
    where the equality follows from the definition of $\EE_{\tPhi^{\pi_t}}^{\pi_t}$ in \eqref{eq:g-rho} and the inequality folllows from the Cauchy-Schwarz inequality.
	We define $\cT^\pi: \sP(\cS\times \cA)\rightarrow \sP(\cS\times\cA)$ as a mapping operator such that $\cT^\pi \tcD(x') = \int \tcD(x)\tP^\pi(x'|\rd x) $.
	We rewrite \eqref{eq:w2-1.1-3} as
	\begin{align}\label{eq:w2-1.1-3rewrite}
		&-\EE_{\tPhi^{\pi_t}}^{\pi_t} \bigl[ \partial_s \bigl(Q(x; \talpha_t^s) - \gamma \cdot Q(x'; \talpha_t^s) \bigr) \partial_s Q(x; \talpha_t^s)\bigr] 
		\le -f(\tphi^{\pi_t})^2 +\gamma \cdot f(\tphi^{\pi_t}) \cdot   f(\cT^{\pi_t}\tphi^{\pi_t}).
	\end{align}
	By the definition of $\cT^{\pi_t}$ and the definition of visitation measure in \eqref{eq:visitation}, it holds that $\tcE_{\tphi^{\pi_t}}^{\pi_t}-\gamma \tcE_{\cT^{\pi_t}\tphi^{\pi_t}}^{\pi_t}=(1-\gamma)\tphi^{\pi_t}$.
	Hence, we have $f(\tphi^{\pi_t})^2-\gamma f^2(\cT^{\pi_t}\tphi^{\pi_t})=(1-\gamma)\EE_{\tphi^{\pi_t}}\big[\big(\partial_s Q(x;\talpha_t^s)\big)^2\big]$ and it holds for \eqref{eq:w2-1.1-3rewrite} that
	\begin{align}\label{eq:w2-1.1-4}
		-f(\tphi^{\pi_t})^2 +\gamma f(\tphi^{\pi_t}) \cdot   f(\cT^{\pi_t}\tphi^{\pi_t})
		&=-\frac{f(\tphi^{\pi_t})}{f(\tphi^{\pi_t})+\gamma f(\cT^{\pi_t}\tphi^{\pi_t})}\cdot \big( f(\tphi^{\pi_t})^2-\gamma^2 f^2(\cT^{\pi_t}\tphi^{\pi_t}) \big)\nonumber\\
		&\le -\frac{f(\tphi^{\pi_t})^2-\gamma \bigg(f(\tphi^{\pi_t})^2-(1-\gamma)\EE_{\tphi^{\pi_t}}\Big[\big(Q(x;\talpha_t^s)\big)^2\Big]
			\bigg)}{1+\sqrt{\gamma}} \nonumber\\
		&\le -(1-\sqrt{\gamma})\cdot\EE_{\tphi^{\pi_t}}\Big[\big(\partial_s Q(x;\talpha_t^s)\big)^2\Big],
	\end{align}
	where the first inequality holds by noting that $f(\tphi^{\pi_t})\ge \sqrt\gamma f(\cT^{\pi_t}\tphi^{\pi_t})$ and the last inequality holds by noting that $f(\tphi^{\pi_t})^2\ge(1-\gamma)\EE_{\tphi^{\pi_t}}\big[\big(\partial_s Q(x;\talpha_t^s)\big)^2\big]$.
	Combining \eqref{eq:w2-1.1-2}, \eqref{eq:w2-1.1-3rewrite}, and \eqref{eq:w2-1.1-4} together, it holds for the first term of \eqref{eq:w2-1.1} that
	\begin{align}\label{eq:w2-1.1.1}
		\eta \cdot \int_0^1 \inp[\big]{\partial_s g(\cdot; \talpha_t^s, \pi_t)}{v_s}_{\talpha_t^s} \rd s
		&\le - \alpha^{-2}\eta(1-\sqrt{\gamma})\cdot \int_0^{1} \EE_{\tphi^{\pi_t}}\Big[\big(\partial_s Q(x;\talpha_t^s)\big)^2\Big] \rd s \nonumber \\
		&\le - \alpha^{-2}\eta(1-\sqrt{\gamma})\cdot \norm[\big]{Q(x;\rho_{\pi_t})-Q(x;\rho_t)}_{2,\tphi^{\pi_t}}^2,
	\end{align}
	where the last inequality holds by the Cauchy-Schwarz inequality.
	For the second term of \eqref{eq:w2-1.1}, we have
	\begin{align}\label{eq:w2-1.1.2}
		\eta \cdot\int_0^1 \int g(\theta; \talpha_t^s, \pi_t) \cdot \partial_s(v_s\talpha_t^s)(\theta) \rd \theta \rd s
		&= \eta \cdot\int_0^1 \int  \inp[\big]{\nabla g(\theta; \talpha_t^s, \pi_t)}{ \talpha_t^s(\theta)\cdot v_s(\theta) \otimes v_s(\theta) }  \rd \theta \rd s\nonumber\\
		&\le  \eta \cdot\int_0^1 \sup_\theta \norm{\nabla g(\theta; \talpha_t^s, \pi_t)}_{\rf} \cdot W_2(\rho_t, \rho_{\pi_t})^2 \rd s\nonumber\\
		& = \eta \cdot \sup_{\theta,s} \norm{\nabla g(\theta; \talpha_t^s, \pi_t)}_{\rf} \cdot W_2(\rho_t, \rho_{\pi_t})^2,
	\end{align}
	where the first equality holds by Eulerian representation of the geodesic in Lemma \ref{lem:euler} that $\partial_s(v_s\cdot \talpha_t^s)=-\Div(\talpha_t^s \cdot v_s \otimes v_s)$ and Stokes's fomula. Here, we denote by $\otimes$ the outer product between two vectors. The inequality of \eqref{eq:w2-1.1.2} follows from $\norm{v_s(\theta) \otimes v_s(\theta)}_{\rf}=\norm{v_s(\theta)}^2$ and the property of the geodesic in \eqref{eq:constant-speed} that $\norm{v_s}_{\talpha_t^s,W_2}=W_2(\rho_t,\rho_{\pi_t})$.
	
	For the second term on the right-hand side of \eqref{eq:w2-1}, we denote by $u_t$  the corresponding vector field at $\rho_{\pi_t}$ such that $\partial_t \rho_{\pi_t}=-\Div(\rho_{\pi_t} u_t)$. Then, it holds that
	\begin{align}\label{eq:w2-1.2}
		- \inp{\dot \rho_{\pi_t}}{\dot{\tbeta_t^0}}_{\rho_{\pi_t},W_2}
	 	= \inp{u_t}{v_1}_{\rho_{\pi_t}} 
	 	\le \norm{\dot \rho_{\pi_t}}_{\rho_{\pi_t},W_2} \cdot W_2(\rho_t, \rho_{\pi_t})
	 	= \normf{\dot\rho_{\pi_t}}_{W_2}\cdot W_2(\rho_t, \rho_{\pi_t}),
	\end{align}
	where the first equality follows from \eqref{eq:w-inner}, the inequality follows from the Cauchy-Schwarz inequality and the facts that $\norm{u_t}_{\rho_{\pi_t}}=\norm{\dot\rho_{\pi_t}}_{\rho_{\pi_t},W_2}$ by \eqref{eq:w-inner} and that $\norm{v_1}_{\rho_{\pi_t}}=W_2(\rho_t,\rho_{\pi_t})$ by \eqref{eq:constant-speed}.
	The last equality of \eqref{eq:w2-1.2} holds by \eqref{eq:MN=VN}.
	Plugging the definition of metric derivative $\normf{\dot\rho_t}_{W_2}$ in \eqref{eq:metric derivative} into \eqref{eq:w2-1.2}, we have
	\begin{align}\label{eq:w2-1.2-0}
		- \inp{\dot \rho_{\pi_t}}{\dot{\tbeta_t^0}}_{\rho_{\pi_t},W_2}
		&\le \lim_{\Delta t\rightarrow 0}\frac{W_2(\rho_{\pi_t},\rho_{\pi_{t+\Delta t}})}{\normf{\Delta t}}\cdot W_2(\rho_t, \rho_{\pi_t})\nend
		&\le \alpha^{-1/2} \cB \cdot \lim_{\Delta t\rightarrow 0} \sup_{s \in \cS}  \EE_{s' \sim \cE_{s}^{\pi_t}} \Big[\norm[\Big]{\frac{\pi_{t+\Delta t}(\cdot\given s')-\pi_t(\cdot\given s')}{\Delta t}}_1 \Big] \cdot W_2(\rho_t, \rho_{\pi_t}) \nonumber\\
		& = \alpha^{-1/2}\cB \cdot 
		\sup_{s \in \cS}\EE_{s' \sim \cE_{s}^{\pi_t}} \Big[\norm[\big]{A_t(s',\cdot) \pi_t(\cdot\given s')}_1 \Big]
		\cdot W_2(\rho_t, \rho_{\pi_t}),
	\end{align}
	where the second inequality follows from Property (iv) of Lemma \ref{lem:full-RAVF} and the equality follows from the MF-PPO update in \eqref{eq:ppo-cont}.
	For the approximation of the advantage function $A_{t}$, it holds that
	\begin{align}\label{eq:A-Q}
		\sup_{x\in \cS\times\cA}\normf[\big]{A_t(x)}&=\sup_{x\in \cS\times\cA}\bigg|Q(x;\rho_t)-\int Q(s,a') \pi_t(\rd a'\given s) \bigg|\nonumber\\
		&\le 2 \sup_{x\in \cS\times\cA} \normf[\big]{Q(x;\rho_t)}.
	\end{align}
	Plugging \eqref{eq:A-Q} into \eqref{eq:w2-1.2-0}, we have
	\begin{align}\label{eq:w2-1.2-1}
		- \inp{\dot \rho_{\pi_t}}{\tbeta_t^0}_{\rho_{\pi_t}}
		& \le  \alpha^{-1/2} \cB
		\cdot 
		\sup_{x\in \cX} \normf[\big]{A_t(x)}
		\cdot
		\sup_{s\in \cS} \EE_{s' \sim \cE_{s}^{\pi_t}} 
		\Big[\norm[\big]{\pi_t(\cdot\given  s')}_1 \Big]
		\cdot W_2(\rho_t, \rho_{\pi_t}) \nonumber\\
		& \le 2 \alpha^{-1/2}\cB \cdot \sup_{x\in \cX}{\normf[\big]{Q(x;\rho_t)}}\cdot W_2(\rho_t, \rho_{\pi_t}).
	\end{align}
	where the last inequality follows from $\norm[]{\pi_t(\cdot\given s')}_1=1$. 
    By first plugging \eqref{eq:w2-1.1.1} and  \eqref{eq:w2-1.1.2} into \eqref{eq:w2-1.1}, and then plugging \eqref{eq:w2-1.1} and  \eqref{eq:w2-1.2-1} into \eqref{eq:w2-1}, we have
	\begin{align}\label{eq:w2-2}
		\frac{\rd}{\rd t} \frac{W_2^2(\rho_t,\rho_{\pi_t})}{2}
		&\le- \alpha^{-2}\eta(1-\sqrt\gamma) \cdot \norm[\big]{Q_t - Q^{\pi_t}}_{2, \tphi^{\pi_t}}^2
		+\eta \cdot \sup_{\theta,s} \norm[\big]{\nabla g(\theta; \talpha_t^s, \pi_t)}_{\rf} \cdot W_2(\rho_t, \rho_\pi)^2 \nonumber\\
		&\quad+2 \alpha^{-1/2}\cB  \cdot \sup_{x\in \cX}{\normf[\big]{Q(x;\rho_t)}}\cdot W_2(\rho_t, \rho_{\pi_t}).
	\end{align}
	Plugging Lemma \ref{lem:boundness} into \eqref{eq:w2-2}, we have
	\begin{align}	\label{eq:w2-3}
		\frac{\rd}{\rd t} \frac{W_2^2(\rho_t,\rho_{\pi_t})}{2}	
		\le &-\eta \alpha^{-2}(1-\sqrt\gamma) \cdot \norm[\big]{Q_t - Q^{\pi_t}}_{2, \tphi^{\pi_t}}^2 \nonumber\\
		&+\eta \alpha^{-1} B_2 \cdot \big(2\alpha B_1 \sup_{s\in[0,1]}W_2(\talpha_t^s, \rho_0) + B_r\big) \cdot W_2(\rho_t, \rho_\pi)^2\nonumber\\
		&+2 \alpha^{1/2} \cB B_1 \cdot W_2(\rho_t, \rho_0) W_2(\rho_t, \rho_{\pi_t}).
	\end{align}
	Note that $\talpha_t^s$ is the geodesic connecting $\rho_t$ and $\rho_\pi$. By Lemma \ref{lem:w2-triangle}, we have
	\begin{align}\label{eq:w2-4}
		\sup_{s\in[0,1]}W_2(\talpha_t^s, \rho_0)\le 2\max\big\{W_2(\rho_{\pi_t}, \rho_0), W_2(\rho_t,\rho_0)\big\}.
	\end{align}
    Plugging \eqref{eq:w2-4} into \eqref{eq:w2-3}, it follows that
	\begin{align*}
		\frac{\rd}{\rd t} \frac{\tW_2^2(\rho_t,\rho_{\pi_t})}{2\eta}	
		\le &- (1-\sqrt\gamma) \cdot \norm[\big]{Q_t - Q^{\pi_t}}_{2, \tphi^{\pi_t}}^2
		+2  \eta^{-1}\alpha^{1/2}\cB B_1 \cdot \tW_2(\rho_t, \rho_0) \tW_2(\rho_t, \rho_{\pi_t})\\
		&+\alpha^{-1} B_2 \cdot \Big(4 B_1 \max\big\{\tW_2(\rho_{\pi_t}, \rho_0), \tW_2(\rho_t,\rho_0)\big\} + B_r\Big) \tW_2(\rho_t, \rho_\pi)^2,
	\end{align*}
	Where $\tW_2=\alpha^{-1}W_2$ is the scaled $W_2$ metric.
	Thus, we complete the proof of Theorem \ref{thm:td}.
\end{proof}

\subsection{Proof of Theorem \ref{thm:ttac}}\label{sec:pf-thm-ttac}
\begin{proof}
	We remark that the restarting mechanism produces discontinuity in $\rho_t$ while $\pi_t$ remains continuous.
	Let $T_0,T_1,\cdots,T_N$ denote the restarting points in $[0,T)$, where $T_0=0$ and $N$ is the total restarting number in $[0,T)$.
	Let $T_n^-$ and $T_n^+$ denote the moments just before and after the restarting occurring at $T_n$, respectively.
	According to the restarting mechanism, we have $\tW_2(\rho_t,\rho_0) \le \lambda \bar D$, $\tW_2(\rho_{T_n^-},\rho_0)=\lambda \bar D$ and $\rho_{T_n^+}=\rho_0$.
	Recall that we set $\tPhi^{\pi_t}=\tcE_{\tphi^{\pi_t}}^{\pi_t}$. By \eqref{eq:td} of Theorem \ref{thm:td}, it holds that
	\begin{align}\label{eq:td-1}
		(1-\sqrt\gamma) \cdot \norm{Q_t - Q^{\pi_t}}_{2, \tphi^{\pi_t}}^2
		\le& -	\frac{\rd}{\rd t} \frac{\tW_2^2(\rho_t,\rho_{\pi_t})}{2\eta}
		+2 \eta^{-1}\alpha^{1/2}\cB  B_1 \lambda \bar D \cdot  \tW_2(\rho_t, \rho_{\pi_t})\nonumber\\
		&+\alpha^{-1} B_2 \cdot (4 B_1 \lambda \bar D + B_r) \tW_2(\rho_t, \rho_\pi)^2.
	\end{align}
	For simplicity, we let $S_1=\frac{(1+\lambda)^2 \bar D^2 B_2(4B_1 \lambda \bar D+B_r)}{1-\sqrt\gamma}$, $ S_2=\frac{2\cB B_1 \lambda(1+\lambda) \bar D^2}{1-\sqrt\gamma}$, $\xi=\frac{1}{2 (1-\sqrt\gamma)}$, $Q(t)=\norm{Q_t-Q^{\pi_t}}_{2, \tphi^{\pi_t}}$, and $\tW(t)=\tW_2(\rho_t,\rho_{\pi_t})$. 
	By sum of the integrals of \eqref{eq:td-1} on $[T_0, T_1],\cdots,[T_{N-1},T_N]$, and $[T_N,T]$, we have
	\begin{align}\label{eq:Qerr1}
		\frac 1T\int_0^T Q^2(t) \rd t 
		&\le \frac 1T\int_0^T \Big(\frac{\alpha^{-1} S_1 }{(1+\lambda)^2\bar D^2}\tW(t)^2 +\frac{\alpha^{1/2}\eta^{-1}S_2}{(1+\lambda)\bar D} \tW(t)\Big) \rd t\nonumber\\
		&\quad-\frac{\xi}{T\eta}\biggl( \sum_{n=0}^{N-1} \int_{T_n}^{T_{n+1}}\frac{\rd}{\rd t} \tW(t)^2 \rd t + \int_{T_N}^T \frac{\rd}{\rd t} \tW(t)^2 \rd t\biggr).
	\end{align}
	Note that we have $\tW(t)=\tW_2(\rho_t,\rho_{\pi_t})\le \tW_2(\rho_t,\rho_0)+\tW_2(\rho_{\pi_t},\rho_0)\le (1+\lambda) \bar D$ by the triangle inequality of $W_2$ distance, the restarting mechanism and Property (iii) of Lemma \ref{lem:full-RAVF}. It thus holds for \eqref{eq:Qerr1} that
	\begin{align}\label{eq:Qerr2}
		\frac 1T\int_0^T Q^2(t) \rd t &\le \frac 1T\int_0^T (\alpha^{-1} S_1 +\alpha^{1/2}\eta^{-1} S_2 ) \rd t\nonumber\\
		&\quad-\frac{\xi}{T\eta}\biggl( \sum_{n=0}^{N-1} \big(\tW^2(T_{n+1}^-)-\tW^2(T_{n}^+)\big) + \big(\tW^2(T)-\tW^2(T_{N}^+)\big) \biggr).
	\end{align}
	Note that we have $\tW(T_{n+1}^-)
	\ge \tW_2(\rho_0,\rho_{T_{n+1}^-})-\tW_2(\rho_0,\rho_{\pi_{T_{n+1}^-}}) \ge (\lambda -1) \bar D 
	\ge \bar D
	\ge \tW(\rho_0,\rho_{\pi_{T_n^+}})=\tW(T_n^+)$ by the triangle inequality of $W_2$ distance, the restarting mechanism and Property (iii) of Lemma \ref{lem:full-RAVF}. It thus holds for \eqref{eq:Qerr2} that
	\begin{align}\label{eq:Qerr3}
		\frac 1T\int_0^T Q^2(t) \rd t \le\alpha^{-1} S_1+\alpha^{1/2} \eta^{-1} S_2 
		+\frac{\eta^{-1} \bar D^2 }{2  T(1-\sqrt\gamma)}.
	\end{align}
	By setting $\tphi^{\pi_t}=\frac{1}{2} \tphi_0+ \frac{1}{2}\phi_0 \otimes \pi_t$ and plugging \eqref{eq:Qerr3} into \eqref{eq:conv-ppo} of Theorem \ref{thm:ppo}, we have
	\begin{align}
		\label{eq:TRE}
		\frac 1T\int_0^T \bigl( J(\pi^*) - J(\pi_t) \bigr) \rd t
		&\le  \frac \zeta T + \frac {4\kappa}{T}\int_0^T Q(t) \rd t \nonumber\\
		&\le \frac \zeta T
		+4\kappa \sqrt{\frac 1T \int_0^T Q^2(t) \rd t}\nonumber\\ 
		&\le \frac \zeta T
		+4\kappa \sqrt{\alpha^{-1} S_1+\alpha^{1/2} \eta^{-1} S_2 
		+\frac{\eta^{-1} \bar D^2 }{2  T(1-\sqrt\gamma)}},
	\end{align}
	where $\kappa= \norm[\big]{ \tcE_{\cD_0}^{\pi^* }/\tphi_0}_{\infty}$ is the concentrability coefficient and $\zeta=\EE_{s\sim \cE_{\pi^*}}\bigl[ \kl\bigl(\pi^*(\cdot \given s) \,\|\, \pi_0(\cdot \given s) \bigr) \bigr]$ is the KL-divergence between $\pi^*$ and $\pi_0$. Here the second inequality follows from the Cauchy-Schwarz inequality.
	Therefore, we complete the proof of \eqref{eq:ttac-error} in Theorem \ref{thm:ttac}.
	In what follows, we aim to upper bound the total restarting number $N$.
	Recall that we have $\tW(T_{n}^+)\le \bar D$ and $\tW(T_{n+1}^-)\ge (\lambda-1)\bar D$ according to the restarting mechanism. Thus,  it holds for \eqref{eq:Qerr2} that
	\begin{align*}
		\frac 1T\int_0^T Q^2(t) \rd t
		&\le \alpha^{-1} S_1 +\alpha^{1/2}\eta^{-1} S_2
		-\frac{\xi}{T\eta}\big( N(\lambda-2)\bar D^2 - \bar D^2\big) .
	\end{align*}
	Since $Q^2(t)\ge 0$, it follows that
	\begin{align*}
		N&\le \big((\alpha^{-1} S_1 +\alpha^{1/2}\eta^{-1} S_2)2T(1-\sqrt\gamma)\eta+\bar D^2\big)\cdot\frac{\bar D^{-2}}{\lambda-2}\nonumber\\
		&=(\lambda-2)^{-1} \big((\alpha^{-1} \eta S_1  +\alpha^{1/2} S_2)2T\bar D^{-2}(1-\sqrt\gamma)+1\big),
	\end{align*}
	which upper bounds the total restarting number $N$. Hence, we complete the proof of Theorem \ref{thm:ttac}.
\end{proof}

%% file: appendix.tex
\section{Proofs of Supporting Lemmas}\label{sec:proofs}
In this section, we give detailed proof of  the supporting lemmas.
\subsection{Generality of Transition Kernel Representation }\label{lem:generality}
Recall that we have the following representation of the transition kernel in Assumption \ref{asp:mdp},
\begin{align}\label{eq:P}
	P(s'\given s, a) &= \int \tilde \sigma \bigl((s, a, 1)^\top w \bigr)  \varphi(s') \psi(s'; w) \rd w,
\end{align}
where $\varphi(s')\ge 0$ and $\tilde\psi(s;\cdot)\in \sP(w)$. 
Such a representation has the same function class as
\begin{align}\label{eq:cP}
	\tP(s'\given s,a)= \int \tilde \sigma\big((s,a,1)^{\top} w\big)  \tilde\psi(s';w)\rd w,
\end{align}
where $\tilde\psi$ is a finite signed measure.
\begin{proof}
For simplicity, let $\cP$ and $\tcP$ denote the function class represented by \eqref{eq:P} and \eqref{eq:cP}, respectively.
Note that for any transition kernel $P(s'\given s,a)$ represented by \eqref{eq:P}, by letting $\tilde\psi(s';w)=\varphi(s') \psi(s';w)$, such a transition kernel $P(s'\given s, a)$ can be equivalently represented by $\tP(s'\given s, a)$ in \eqref{eq:cP}. 
Thus, it holds that $\cP\subseteq \tcP$. Therefore, we only need to prove $\tcP\subseteq\cP$, which is equivalent to proving that for any $\tP(s'\given s,a)\in \tcP$ given by \eqref{eq:cP}, the signed measure $\tilde\psi$ can be non-negative. If that is the case, by letting $\varphi(s')=\int \big(\tilde\psi_+(s'; w) + \tilde\psi_-(s'; -w)\big)\rd w$ and $\psi(s';w)=\varphi(s')^{-1} \bigl(\tilde\psi_+(s'; w) + \tilde\psi_-(s'; - w)\bigr)$, we can have \eqref{eq:cP} equivalently represented by \eqref{eq:P}. Note that there always exist non-negative functions $\tilde\psi_+$ and $\tilde\psi_-$ such that $\tilde\psi=\tilde\psi_+-\tilde\psi_-$. Since $\tilde \sigma$ is an odd function, it holds that
\begin{align*}
	\tP(s' \given s, a) &= \int \tilde \sigma(w^\top x) \tpsi(s'; w) \rd w \\
	& = \int \tilde \sigma(w^\top x) \tilde\psi_+(s'; w) \rd w + \int \tilde \sigma( -w^\top x) \tilde\psi_-(s'; w) \rd w \\
	& = \int \tilde \sigma(w^\top x)\bigl(\tilde\psi_+(s'; w) + \tilde\psi_-(s'; - w)\bigr)\rd w.
\end{align*}
Thus, by letting $\varphi(s')=\int \big(\tilde\psi_+(s'; w) + \tilde\psi_-(s'; -w)\big)\rd w$ and $\psi(s';w)=\varphi(s')^{-1} \bigl(\tilde\psi_+(s'; w) + \tilde\psi_-(s'; - w)\bigr)$, it holds that $\tP(s' \given s, a) = \int \tilde \sigma(w^\top x) \varphi(s')\psi(s'; \rd w)=P(s'\given s,a)$, where $ \psi \in \sP(\cW)$ and $\varphi\ge 0$.
Hence, any $\tP(s'\given s,a)\in \tcP$ can be equivalently represented by \eqref{eq:P} and it follows that $\tcP\subseteq\cP$.
Thus, \eqref{eq:P} has the same function class as \eqref{eq:cP} and we complete the proof.
\end{proof}

\subsection{Proof of Detailed Version of Lemma \ref{lem:RAVF}} \label{sec:pf-lem-RAVF}
In this section, we prove a more detailed version of Lemma \ref{lem:RAVF}, i.e., Lemma \ref{lem:full-RAVF}.
\begin{proof}
	\Vskip
	We begin by a sketch of the proof of Lemma \ref{lem:full-RAVF}. We first construct functions $Z_\pi$ and $\nu_\pi(w)$. With the use of mollifiers, we prove that there exists a function $p_\pi(b)$ satisfying \eqref{eq:def-p_pi} and then formulate a construction of $\bar \rho_\pi(\theta)$, which gives way to obtain $\rho_\pi$.
	For the proof of Property (iii), using the technique of Talagrand's inequality and the chi-squared divergence, we establish a constant upper bound for $W_2(\rho_\pi, \rho_0)$.
	For the proof of Property (iv), by exploiting the inequality between $W_2$ distance and the weighted homogeneous Sobolev norm, we upper bound $W_2(\rho_{\pi_1}, \rho_{\pi_2})$ up to $O(\alpha^{-1/2})$.
	
	\noindent{\bf Proof of Property (i) of Lemma \ref{lem:full-RAVF}.}
	We give a proof of Property (i) by a construction of $\rho_\pi$.
	For notational simplicity, we let $x = (s, a, 1)$. By definitions of the action value function $Q^\pi$ and the state value function $V^\pi$ in \eqref{eq:def-q}, we have
	\begin{align}\label{eq:Q1}
		Q^{\pi}(s, a) &= r(s, a) + \gamma \cdot \int P(s'\given s, a) V^\pi(s') \rd s' \nonumber\\ 
		& = \int \tilde \sigma(w^\top x) \Bigl\{ B_r\cdot \mu(w) + \gamma \cdot \int \varphi(s')\psi(s' ; w) V^{\pi}(s') \rd s' \Bigr\} \rd w \nonumber\\
		& = \int Z_\pi \cdot \tilde \sigma(w^\top x) \nu_\pi(w) \rd w,
	\end{align}
	where 
	\begin{gather}
		\nu_\pi(w) = Z_\pi^{-1} \cdot\Bigl( B_r \mu(w) + \gamma \cdot \int \varphi(s')\psi(s' ; w) V^{\pi}(s') \rd s'\Bigr), \label{eq:nv_pi}\\
 		Z_\pi = B_r + \gamma \cdot \int \varphi(s') V^{\pi}(s') \rd s'. \label{eq:Z_pi}
	\end{gather}
	Here, the second equality in \eqref{eq:Q1} holds by (i) of Assumption \ref{asp:mdp}.
	We construct $\bar \rho_\pi$ by $\bar \rho_{\pi} = \nu_\pi\times p_\pi$, i.e., $\bar \rho_{\pi}(w, b) = \nu_\pi(w)p_\pi(b)$, where $p_\pi(b)$ is defined to be a probability measure in $ \sP_2(\RR)$ such that 
	\begin{align}\label{eq:def-p_pi}
		\int B_\beta \cdot \beta(b) p_\pi(\rd b) = Z_\pi.
	\end{align}
	We remark that such a $p_\pi$ exists and we will provide a construction later.
	Since we have $V^\pi\ge 0$, $\varphi\ge 0$, and $\psi\ge 0$ by Assumption \ref{asp:mdp},  it turns out that $\nu_\pi$ is a probability density function according to \eqref{eq:nv_pi}, which further suggests that $\bar \rho_\pi\in \sP_2(\RR^D)$.
	Plugging \eqref{eq:def-p_pi} into \eqref{eq:Q1}, it holds that
	\begin{align*}
		Q^\pi(x) = \int \sigma(x; \theta) \bar \rho_\pi(\theta) \rd \theta,
	\end{align*}
	where the equality holds by noting that $\sigma(x;\theta)=B_\beta \cdot \beta(b) \cdot \tilde\sigma(w^\top x)$ in \eqref{eq:nn} and that $\bar \rho_{\pi}= \nu_\pi\times p_\pi$.
	Furthermore, by letting $\rho_{\pi}=\bar{\rho}_\pi+(1-\alpha^{-1})(\rho_0-\bar{\rho}_\pi)$, we have
	\begin{align*}
		Q^\pi(x)&=\alpha \int \sigma(x;\theta) (\rho_\pi-(1-\alpha^{-1})\rho_0)\rd \theta=\alpha \int \sigma(x;\theta) \rho_\pi \rd \theta=Q(x;\rho_\pi),
	\end{align*}
	where the second equality holds by noting that $\sigma$ is odd with respect to $w$ and $b$ and that $\rho_0\sim \cN(0,I_D)$ is an even function. Thus, we finish the construction of $\rho_\pi$ and also complete the proof of Property (i) in Lemma \ref{lem:full-RAVF}.
	
	\Vskip
	\noindent{\bf A Construction for $p_\pi(b)$.}
	Recall that we have $p_\pi$ defined in \eqref{eq:def-p_pi}.
	Here, we provide a construction for $p_\pi$ which has some properties that will facilitate our analysis. Ideally, we want $p_\pi$ to have global support and concentrate to its mean, which motivates us to consider $p_\pi$ to be Gaussian distribution with high variance.
	Recall that we assume that $\beta^{-1}$ is $\ell_\beta$-Lipschitz continuous on $ [-2/3, 2/3] $ in Assumption \ref{asp:nn}. Let $q(b) $ be the probability density function of the standard Gaussian distribution,i.e., $q\sim \cN(0, 1)$. Then, $ q_{\epsilon}(b - z) = \epsilon^{-1} \cdot q ((b - z)/\epsilon)$ is the probability density function such that $q_\epsilon \sim \cN(z, \epsilon^{2})$. We define function $\beta_\epsilon$ as follows,
	\begin{align}\label{eq:beta-eps}
		\beta_\epsilon(z) &= \int \beta(b) q_{\epsilon}(b - z) \rd b = (\beta * q_\epsilon)(z),
	\end{align}
	where $*$ denotes the convolution. Note that $\{q_\epsilon\}_{\epsilon>0}$ can be viewed as a class of mollifiers \citep{friedrichs1944identity}. 
	In particular, let
	\begin{align}\label{eq:bepsilon}
		\bar\epsilon=\min \Big\{ \sqrt{\frac{\pi}{2}}\cdot\frac{1}{6 L_{0,\beta}}, \sqrt{\frac{\pi}{2}}\cdot\frac{1}{2 \ell_\beta L_{1,\beta}}, 1\Big\},
	\end{align}
	where $L_{0,\beta}$ and $L_{1,\beta}$ characterize the Lipschitz continuity and smoothness of $\beta$ respectively and $\beta^{-1}$ is $\ell_\beta$-Lipschitz continuous in $[-2/3, 2/3]$ by Assumption \ref{asp:nn}.
	For the approximation error of mollifier $\beta_\epsilon$, it holds that
	\begin{align*}
		\big|\beta(z)-\beta_{\bar\epsilon}(z)\big|
		&=\Big|\int \big(\beta(z)-\beta(z')\big)q_{\bepsilon} (z-z')\rd z' \Big|\nonumber\\
		&\le L_{0,\beta}\cdot\int  |z-z'| q_{\bepsilon} (z-z') \rd z' \nonumber\\
		&= L_{0,\beta} \cdot\bar\epsilon\cdot \sqrt{\frac{2}{\pi}},
	\end{align*}
	where the last inequality follows from $\int |z|q_{\epsilon}(z)\rd z=\epsilon\cdot\sqrt{2/\pi}$.
	Similarly, we have $|\dot\beta(z)-\dot\beta_{\bepsilon}(z)|\le L_{1,\beta}\cdot\bepsilon \cdot\sqrt{2/\pi}$. By definition of $\bepsilon$ in  \eqref{eq:bepsilon}, it further holds that
	\begin{gather}
		\sup_{ b\in \beta^{-1}([-2/3, 2/3]) } \bigl| \beta(b) - \beta_{\bar \epsilon}(b) \bigr| \le 1/6,\label{eq:beta-conv}\\ 
		\sup_{ b\in \beta^{-1}([-2/3, 2/3]) } \bigl| \dot\beta(b) - \dot\beta_{\bar \epsilon}(b) \bigr| \le \frac{1}{2\ell_\beta}\label{eq:beta'-conv}.
	\end{gather}
	Note that $\beta(b)$ is a monotonic function with $| \dot\beta(b) | \ge 1/\ell_\beta$ in $\beta^{-1}([-2/3, 2/3])$ by Assumption \ref{asp:nn}. With regard to \eqref{eq:beta'-conv}, it follows that $\beta_{\bar \epsilon}(b)$ is also monotonic in $\beta^{-1}([-2/3, 2/3])$ and that
	\begin{align}\label{eq:ebeta-L0}
		| \dot\beta_{\bar \epsilon}(b) | 
		\ge \normf{\dot\beta(b)}-\normf{\dot\beta(b)-\dot\beta_{\bepsilon}(b)}\ge
		\frac {1}{2\ell_\beta},\quad \forall b\in \beta^{-1}([-2/3, 2/3]).
	\end{align}
	Furthermore, by \eqref{eq:beta-conv} and the continuity of $\beta_{\bepsilon}$ in $\beta^{-1}([-2/3, 2/3])$, we have
	\begin{align}\label{eq:ebeta-value domain}
		[-1/2, 1/2] \subseteq \beta_{\bar \epsilon}(\beta^{-1}([-2/3, 2/3])).
	\end{align}
	The monotonicity of $\beta_{\bepsilon}(b)$, \eqref{eq:ebeta-L0}, and \eqref{eq:ebeta-value domain} together show that $\beta_{\bar \epsilon}^{-1}$ exists  and is $2\ell_\beta$-Lipschitz continuous in $[-1/2, 1/2]$. 
	Moreover, since $\beta$ is an odd function, it holds by \eqref{eq:beta-eps} that $\beta_\bepsilon$ is also an odd function with $\beta_\bepsilon(0)=0$. Hence, it holds that $\beta_\bepsilon^{-1}([-1/2,1/2])\subseteq [-\ell_\beta, \ell_\beta]$.
	Furthermore, by \eqref{eq:Z_pi}, it holds that
	\begin{align*}
		Z_\pi  &= B_r + \gamma \cdot \int \varphi(s') V^{\pi}(s') \rd s' \\
		& \le B_r + \gamma \cdot (1-\gamma )^{-1} \cdot B_r \cdot \int \varphi(s')\rd s' \\
		& \le B_r + \gamma(1-\gamma)^{-1}B_rM_{1, \varphi},
	\end{align*}
	where the first inequality follows from the fact that $V^\pi(s) \le (1-\gamma)^{-1} \cdot \sup_{s, a} r(s, a)$ and the last inequality follows from Assumption \ref{asp:mdp} that $\int \varphi(s')\rd s'\le M_{1,\varphi}$. 
	By setting $ B_\beta \ge 2 (B_r + \gamma(1-\gamma)^{-1}B_rM_{1, \varphi})$, it holds that $|Z_\pi/B_\beta| \le 1/2$, which indicates that $\beta_{\bar \epsilon}^{-1}(Z_\pi/ B_\beta)$ exists and allows $p_\pi(b) = q_{\bar\epsilon}(b- \beta_{\bar \epsilon}^{-1}(Z_\pi/ B_\beta))$ to be the probability density function such that $p_\pi(b)\sim\cN(\beta_{\bar \epsilon}^{-1}(Z_\pi/ B_\beta), \bar \epsilon^2)$. For the mean value $\beta_{\bepsilon}^{-1}(Z_\pi/ B_\beta)$, recalling that $\beta_\bepsilon^{-1}([-1/2,1/2])\subseteq [-\ell_\beta, \ell_\beta]$, it thus holds that
	\begin{align}\label{eq:p_pi-median}
		\normf[\big]{\beta_{\bar \epsilon}^{-1}(Z_\pi/ B_\beta)}\le \ell_\beta.
	\end{align}
 	Following from \eqref{eq:beta-eps}, we have
	\begin{align*}
		\int \beta(b) p_\pi(b) \rd b =\int \beta(b) q_{\bar\epsilon}(b- \beta_{\bar \epsilon}^{-1}(Z_\pi/ B_\beta)) \rd b= \beta_{\bar \epsilon}\bigl(\beta_{\bar \epsilon}^{-1}(Z_\pi/ B_\beta) \bigr) =  Z_\pi/ B_\beta.
	\end{align*}
	Hence, our construction of $p_\pi$ here is in line with the definition of $p_\pi$ in \eqref{eq:def-p_pi}. In the sequel, we consider $p_\pi(b)=q_{\bar{\epsilon}}\big(b-\beta_{\bar{\epsilon}}^{-1}(Z_\pi/B_\beta)\big)$ to hold throughout.
	
	\Vskip
	\noindent{\bf Proof of Property (ii) of Lemma \ref{lem:full-RAVF}.}
	Here we show that $g(\cdot;\pi,\rho_\pi)=0$ is a direct result of $Q^\pi(x)=Q(x;\rho_\pi)$ in Property (i). Note that
	\begin{align}\label{eq:Q^pi prop1}
		Q^\pi(x)-r(x)-\gamma\EE_{x'\sim \tP^\pi(\cdot\given x)} Q^\pi(x')=0
	\end{align}
	holds by the definition of the action value function $Q^\pi$ in \eqref{eq:def-q} for any $x\in \cS\times\cA$. Since we have $Q^\pi(x)=Q(x;\rho_\pi)$ proved, by plugging \eqref{eq:Q^pi prop1} into the definition of $g$ in \eqref{eq:g-rho}, where $Q(\cdot;\rho_\pi)$ is substituted for $Q^\pi$, it follows that $g(x;\pi,\rho_\pi)=0$. Thus, we complete the proof of Property (ii) of Lemma \ref{lem:full-RAVF}.
	
	\noindent{\bf{Proof of Property (iii) of Lemma \ref{lem:full-RAVF}.}}
	
	In what follows, we aim to upper bound $W^2(\rho_{\pi}, \rho_0)$. 
	We summerize our aforementioned constructions as follows,
	\begin{gather}
		Z_\pi = B_r + \gamma \cdot \int \varphi(s') V^{\pi}(s') \rd s',\nend
		\nu_\pi(w) =Z_\pi^{-1}\Big(B_r \mu(w) + \gamma \cdot \int \varphi(s')\psi(s' ; w) V^{\pi}(s') \rd s'\Big),\label{eq:nu-pi}\\
		p_\pi(b)=q_{\bar{\epsilon}}\big(b-\beta_{\bar{\epsilon}}^{-1}(Z_\pi/B_\beta)\big),\label{eq:p-pi}\\
		\bar{\rho}_{\pi}(\theta)=p_\pi(b) \nu_\pi(w),\label{eq:bar-rho}\\
		\rho_{\pi}=\bar{\rho}_\pi+(1-\alpha^{-1})(\rho_0-\bar{\rho}_\pi)\label{eq:rho_pi}.
	\end{gather}
	Plugging \eqref{eq:p-pi} into the definition of Chi-squared divergence, we have
	\begin{align}\label{eq:chi p}
		\chi^2\big(p_\pi \d| \rho_{0,b}\big)
		= \int \frac{p_\pi^2}{\rho_{0,b}} \rd b-1
		=\frac{1}{\bar\epsilon\sqrt{2-\bar\epsilon^2}}\cdot \exp\Big\{\frac{ \big(\beta_{\bar{\epsilon}}^{-1}(Z_\pi/B_\beta)\big)^2}{2-\bar\epsilon^2}\Big\}-1,
	\end{align}
	Note that we have $\bar\epsilon=\min \Big\{ \sqrt{\pi/2}\cdot(6 L_{0,\beta})^{-1}, \sqrt{\pi/2}\cdot(2 \ell_\beta L_{1,\beta})^{-1}, 1\Big\}$ by \eqref{eq:bepsilon} and $\normf[\big]{\beta_{\bar \epsilon}^{-1}(Z_\pi/ B_\beta)}\le \ell_\beta$ by \eqref{eq:p_pi-median}.
	Hence, we have $\chi^2(p_\pi \d| \rho_{0,b})$ upper bounded.
	As for $\nu_\pi$ in \eqref{eq:nu-pi}, we have
	\begin{align}\label{eq:chi nu}
		\chi^2 ( \nu_\pi \d| \rho_{0,w})
		&= \chi^2 \bigg(B_r  Z_\pi^{-1} \mu + \int V^\pi(s')Z_\pi^{-1} \varphi(s')\psi(s';w) \rd s' \,\Big\|\, \rho_{0,w} \bigg) \nonumber\\
		&\le 3\bigg( \chi^2(B_rZ_\pi^{-1} \mu \d| \rho_{0,w})+\chi^2\Big(\int V^\pi(s')Z_\pi^{-1} \varphi(s')\psi(s';\cdot) \rd s' \,\Big\|\, \rho_{0,w}\Big) +1\bigg),
	\end{align}
	where the inequality holds by Property (iii) of Lemma \ref{lem:chi}.
	For the first term on the right-hand side of \eqref{eq:chi nu}, we have
	\begin{align}\label{eq:chi nu-0}
		\Cnorm{}{B_r Z_\pi^{-1} \mu}{\rho_{0,w}} &= B_r^2 Z_\pi^{-2}\Cnorm{}{\mu}{\rho_{0,w}}+(1-B_r Z_\pi^{-1})^2 \nend
		& \le  \Cnorm{}{\mu}{\rho_{0,w}}+1\nend
		&\le M_\mu +1,
	\end{align}
	where the equality holds by  Property (i) of Lemma \ref{lem:chi}, the first inequality holds by noting that $\normf{B_r Z_\pi^{-1}}\le 1$ and the last inequality follows from $\Cnorm{}{\mu}{\rho_{0,w}}<M_\mu$ by Assumption \ref{asp:mdp}. Hence, the first term on the right-hand side of \eqref{eq:chi nu} is upper bounded.
	As for the second term, by Property (iv) of Lemma \ref{lem:chi}, we have
	\begin{align}\label{eq:chi nu-1}
		&\Cnorm[\Big]{\Big}{\int V^\pi(s')Z_\pi^{-1} \varphi(s')\psi(s';\cdot) \rd s'}{\rho_{0,w}} \nonumber \\
		& \le \int \big(V^\pi(s')Z_\pi^{-1} \varphi(s')\big)^2 \rd s' \cdot \int \Cnorm{}{\psi(s';\cdot)}{\rho_{0,w}} \rd s'+ \Big( \int V^\pi(s')Z_\pi^{-1} \varphi(s') \rd s' -1 \Big)^2 \nonumber \\
		&\le (1-\gamma)^{-2} M_{2,\varphi} \cdot M_\psi+
		(1-\gamma)^{-2} M_{1,\varphi}^2 +1 ,
	\end{align}
	where the last inequality holds by noting that $|V^{\pi}|\le (1-\gamma)^{-1}B_r$, $Z_\pi\ge B_r$, $\norm{\cS}\le 1$, and that $\chi^2\big(\psi(s';\cdot)\d|\rho_{0,w}\big)<M_\psi$ by Assumption \ref{asp:mdp}. Hence, it holds that the second term of \eqref{eq:chi nu} is also upper bounded.
	Plugging \eqref{eq:chi nu-0} and \eqref{eq:chi nu-1} into \eqref{eq:chi nu}, we can establish the upper bound for $\Cnorm{}{\nu_\pi}{\rho_{0,w}}$.
	Furthermore, by Property (ii) of Lemma \ref{lem:chi} and noting that $\bar \rho_\pi =p_\pi \times \nu_\pi$, we have $\Cnorm{}{\bar \rho_\pi}{\rho_0}$ upper bounded as well, that is, 
	\begin{align*}
		\chi^2(\bar{\rho}_\pi \d| \rho_0)<\frac{1}{2}\bar{D}^2,
	\end{align*}
	where $\bar D$ depends on absolute constants occurring in \eqref{eq:chi p}, \eqref{eq:chi nu-0}, and \eqref{eq:chi nu-1}, i.e., $\ell_\beta$, $L_{0,\beta}$, $L_{1,\beta}$, $B_r$, $M_{1,\varphi}$, $M_{2,\varphi}$,  $M_\mu$, and $M_{\psi}$ in Assumptions \ref{asp:nn} and \ref{asp:mdp}.
	Since $\rho_0 \sim \mathcal{N}(0,I_D)$, it holds for any $\rho_\pi$ that,
	\begin{align}\label{eq:w2-D}
		\frac{1}{2}W_2(\rho_\pi,\rho_0)^2 
		&\le \kl(\rho_{\pi}\d|\rho_0)
		\le \int \Big(\frac{\rho_\pi}{\rho_0}-1\Big) \frac{\rho_\pi}{\rho_0} \rho_0(\rd\theta)
		=\int \Big(\frac{\rho_\pi}{\rho_0}-1\Big)^2  \rho_0(\rd\theta) \nend
		&=\int \Big(\frac{(1-\alpha^{-1})\rho_0(\theta)+\alpha^{-1}\bar{\rho}_\pi(\theta)}{\rho_0(\theta)}-1\Big)^2 \rho_0(\rd\theta) 
		=\alpha^{-2} \chi^2(\bar{\rho}_\pi \d| \rho_0)< \frac{\alpha^{-2}}{2}\bar{D}^2,
	\end{align}
	where the first inequality follows from Talagrand's inequality in Lemma \ref{lem:talagrand}.
	Plugging $\tW_2=\alpha W_2$ into \eqref{eq:w2-D}, we complete the proof of Property (iii) of Lemma \ref{lem:full-RAVF}.
	
	\Vskip
	\noindent{\bf{Proof of Property (iv) of Lemma \ref{lem:full-RAVF}}.}
	\Vskip
	\noindent{\bf{Upper Bounding $W_2(p_{\pi_1},p_{\pi_2})$}.}
	Following the property of $W_2$ distance with respect to Gaussian distribution, we have  $ W_2(p_{\pi_1}, p_{\pi_2}) = \big| \beta_{\bar \epsilon}^{-1}(Z_{\pi_1}/ B_\beta) - \beta_{\bar \epsilon}^{-1}(Z_{\pi_2}/ B_\beta) \big| $. 
	Recall that we have $|Z_\pi/B_\beta| \le {1}/{2}$ and that $\beta_\bepsilon^{-1}$ is $2\ell_\beta$-Lipschitz continuous on $[-1/2,1/2]$. It then holds that
	\begin{align}\label{eq:wp1}
		W_2(p_{\pi_1}, p_{\pi_2}) = \bigl| \beta_{\bar \epsilon}^{-1}(Z_{\pi_1}/ B_\beta) - \beta_{\bar \epsilon}^{-1}(Z_{\pi_2}/ B_\beta) \bigr| \le 2\ell_\beta \cdot |Z_{\pi_1}- Z_{\pi_2}|/ B_\beta. 
	\end{align}
	Meanwhile, we have
	\begin{align} \label{eq:diff-z}
		|Z_{\pi_1}- Z_{\pi_2}| & \le \gamma \cdot \int \varphi(s') \cdot \bigl| V^{\pi_1}(s') - V^{\pi_2}(s') \bigr| \rd s' \nonumber\\
		& \le \gamma \cdot \int \varphi(s') \rd s' \cdot \sup_{s'\in \cS} \bigl| V^{\pi_1}(s') - V^{\pi_2}(s') \bigr| \nonumber\\
		&\le \gamma \cdot M_{1, \varphi} \cdot \sup_{s'\in \cS}\bigl| V^{\pi_1}(s') - V^{\pi_2}(s') \bigr|,
	\end{align}
	where the last inequality holds by (ii) of Assumption \ref{asp:mdp}.
	Plugging \eqref{eq:diff-z} into \eqref{eq:wp1}, it holds for $W_2(p_{\pi_1}, p_{\pi_2})$ that
	\begin{align}
		\label{eq:wp2}
		W_2(p_{\pi_1}, p_{\pi_2}) \le \frac{2\ell_\beta \cdot \gamma \cdot M_{1, \varphi} \cdot \sup_{s'\in \cS}\bigl| V^{\pi_1}(s') - V^{\pi_2}(s') \bigr|}{B_\beta}.
	\end{align}
	
	\Vskip
	\noindent{\bf Upper Bounding $ W_2(\nu_{\pi_1}, \nu_{\pi_2}) $.}
	By definition of $\nu_\pi$ in \eqref{eq:nu-pi}, we have
	\begin{align}\label{eq:nu-1}
		\nu_{\pi_1}-\nu_{\pi_2}
		=\frac{B_r(Z_{\pi_2}-Z_{\pi_1})}{Z_{\pi_1}Z_{\pi_2}}\mu+\gamma \int \varphi(s') \psi(s';\cdot) \Big(\frac{V^{\pi_1}}{Z_{\pi_1}}-\frac{V^{\pi_2}}{Z_{\pi_2}}\Big)\rd s'.
	\end{align}
	For $V^{\pi_1}/Z_{\pi_1}-V^{\pi_2}/Z_{\pi_2}$, it holds that
	\begin{align}\label{eq:nu-1.1}
		\biggl|\frac{V^{\pi_1}}{Z_{\pi_1}}-\frac{V^{\pi_2}}{Z_{\pi_2}}\biggr|
		&\le \max\{V^{\pi_1},V^{\pi_2}\}\cdot \frac{|Z_{\pi_1}-Z_{\pi_2}|}{Z_{\pi_1}Z_{\pi_2}}
		+\max \{Z^{-1}_{\pi_1},Z^{-1}_{\pi_2}\} |V^{\pi_1}-V^{\pi_2}| \nonumber \\
		&\le (B_r^{-1} (1-\gamma)^{-1} \gamma M_{1,\varphi}+B_r^{-1}) \sup_{s'\in S} \normf[\big]{V^{\pi_1}(s')-V^{\pi_2}(s')},
	\end{align}
	where the last inequality holds by noting that
	\begin{align}\label{eq:nu-1.2}
		\bigg|\frac{Z_{\pi_2}-Z_{\pi_1}}{Z_{\pi_1}Z_{\pi_2}}\bigg|\le \gamma M_{1,\varphi} B_r^{-2}\sup_{s'\in S} \normf[\big]{V^{\pi_1}(s')-V^{\pi_2}(s')}.
	\end{align}
	Here the inequality in \eqref{eq:nu-1.2} holds by \eqref{eq:diff-z} and the fact that $Z_\pi\ge B_r$.
	For $W_2(\nu_{\pi_1},\nu_{\pi_2})$, by Lemma \ref{lem:w2-ub}, it holds that
	\begin{align}\label{eq:wnu2-1}
		W_2(\nu_{\pi_1},\nu_{\pi_2})\le 2\norm{\nu_{\pi_1}-\nu_{\pi_2}}_{\dot H^{-1}(\nu_{\pi_1})}.
	\end{align}
	Recall that we have $\nu_\pi$ defined in \eqref{eq:nu-pi} that
	\begin{gather*}
		\nu_\pi(w) = Z_\pi^{-1} \cdot\Bigl( B_r \mu(w) + \gamma \cdot \int \varphi(s')\psi(s' ; w) V^{\pi}(s') \rd s'\Bigr), \nend
		Z_\pi = B_r + \gamma \cdot \int \varphi(s') V^{\pi}(s') \rd s'
	\end{gather*}
	where $Z_\pi^{-1}(B_r+\gamma\cdot\int\varphi(s')V^\pi(s')\rd s')=1$, $B_rZ_\pi^{-1}\ge 0$, and $\gamma\varphi(s')V^\pi(s')Z_\pi^{-1}\ge 0$, which indicate that $\nu_\pi$ is in the convex hull of $\mu$ and $\psi(s';\cdot)$. Hence, by Property (i) of Lemma \ref{lem:G}, it holds that
	\begin{align}\label{eq:wnu2-2}
		2\norm{\nu_{\pi_1}-\nu_{\pi_2}}_{\dot H^{-1}(\nu_{\pi_1})}\le 2 \max \big\{ \sup_{s} \Hnorm{\nu_{\pi_1}-\nu_{\pi_2}}{\psi(s;\cdot)}, \Hnorm{\nu_{\pi_1}-\nu_{\pi_2}}{\mu} \big\}.
	\end{align}
	Furthermore, following from \eqref{eq:nu-1} and $Z_\pi^{-1}(B_r+\gamma\cdot\int\varphi(s')V^\pi(s')\rd s')=1$, by Property (ii) of Lemma \ref{lem:G}, it holds for $2\Hnorm{\nu_{\pi_1}-\nu_{\pi_2}}{\mu}$ that
	\begin{align}\label{eq:wnu2-3}
		2\Hnorm{\nu_{\pi_1}-\nu_{\pi_2}}{\mu}&\le \max\{\sup_{s'\in \cS}\Hnorm{\mu-\psi(s';\cdot)}{\mu} , \sup_{(s',s'')\in \cS\times\cS} \Hnorm{\psi(s';\cdot)-\psi(s'';\cdot)}{\mu}\}\nend 
		&\quad\cdot\bigg(\Big|\frac{B_r(Z_{\pi_2}-Z_{\pi_1})}{Z_{\pi_1}Z_{\pi_2}}\Big|+\gamma \int \varphi(s') \Big|\frac{V^{\pi_1}}{Z_{\pi_1}}-\frac{V^{\pi_2}}{Z_{\pi_2}}\Big|\rd s'\bigg).
	\end{align}
	Plugging (iii) of Assumption \ref{asp:mdp}, \eqref{eq:nu-1.1}, and \eqref{eq:nu-1.2} into \eqref{eq:wnu2-3}, we have
	\begin{align}\label{eq:wnu2-4}
		2\Hnorm{\nu_{\pi_1}-\nu_{\pi_2}}{\mu}&\le \cG \cdot \Big(
		\gamma M_{1,\varphi} B_r^{-1}
		+\gamma M_{1,\varphi}\big(B_r^{-1} (1-\gamma)^{-1} \gamma M_{1,\varphi}+B_r^{-1}\big)\Big)\nend
		&\quad\cdot\sup_{s'\in S} \normf[\big]{V^{\pi_1}(s')-V^{\pi_2}(s')}.
	\end{align}
	By simply substituting $\Hnorm{\cdot}{\psi(s;)}$ for $\Hnorm{\cdot}{\nu}$ in both \eqref{eq:wnu2-3} and \eqref{eq:wnu2-4}, it also holds for $2\Hnorm{\nu_{\pi_1}-\nu_{\pi_2}}{\psi(s;\cdot)}$ that
	\begin{align}\label{eq:wnu2-5}
		2\Hnorm{\nu_{\pi_1}-\nu_{\pi_2}}{\psi(s;\cdot)}&\le \cG \cdot \Big(
		\gamma M_{1,\varphi} B_r^{-1}
		+\gamma M_{1,\varphi}\big(B_r^{-1} (1-\gamma)^{-1} \gamma M_{1,\varphi}+B_r^{-1}\big)\Big)\nend
		&\quad\cdot\sup_{s'\in S} \normf[\big]{V^{\pi_1}(s')-V^{\pi_2}(s')}.
	\end{align}
	Combining \eqref{eq:wnu2-1}, \eqref{eq:wnu2-2}, \eqref{eq:wnu2-4},  and \eqref{eq:wnu2-5}, we have
	\begin{align}\label{eq:wnu2-6}
		W_2(\nu_{\pi_1},\nu_{\pi_2})&\le \cG \cdot \Big(
		\gamma M_{1,\varphi} B_r^{-1}
		+\gamma M_{1,\varphi}\big(B_r^{-1} (1-\gamma)^{-1} \gamma M_{1,\varphi}+B_r^{-1}\big)\Big)\nend
		&\quad\cdot\sup_{s'\in S} \normf[\big]{V^{\pi_1}(s')-V^{\pi_2}(s')}.
	\end{align}
	\Vskip
	\noindent{\bf{Upper Bounding $W_2({\rho}_{\pi_1},{\rho}_{\pi_2})$}.} 
	Note that we have $\bar\rho_\pi=\nu_\pi\times p_\pi$ in \eqref{eq:bar-rho}. By Lemma \ref{lem:w-ind}, we have
	\begin{align}
		\label{eq:ubwr2}
		W_2(\bar{\rho}_{\pi_1},\bar{\rho}_{\pi_2})
		\le \sqrt{W_2^2 (\nu_{\pi_1}, \nu_{\pi_2}) + W_2^2 (p_{\pi_1}, p_{\pi_2})} 
		\le \cB' \sup_{s'\in S} \normf[\big]{V^{\pi_1}(s')-V^{\pi_2}(s')},
	\end{align}
	where $\cB'$ depends on the discount factor $\gamma$ and absolute constants $\ell_\beta, B_\beta, B_r, M_{1,\varphi}, \cG$ in Assumption \ref{asp:nn} and \ref{asp:mdp} according to \eqref{eq:wp2} and \eqref{eq:wnu2-6}. 
	By performance difference lemma \citep{kakade2002approximately}, we have
	\begin{align}\label{eq:V-diff}
		\bigl| V^{\pi_1}(s) - V^{\pi_2}(s) \bigr| 
		&= (1-\gamma)^{-1} \cdot \biggl| \EE_{s' \sim\cE_s^{\pi_1}}\Bigl[ \inp[\big]{A^{\pi_2}(s',\cdot)}{\pi_1(\cdot\given s') - \pi_2(\cdot \given s')}_\cA \Bigr] \biggr| \nend
		&= (1-\gamma)^{-1} \cdot \biggl| \EE_{s' \sim\cE_s^{\pi_1}}\Bigl[ \inp[\big]{Q^{\pi_2}(s', \cdot )}{\pi_1(\cdot\given s') - \pi_2(\cdot \given s')}_\cA \Bigr] \biggr| \nend
		&\le (1-\gamma)^{-2} 
		\cdot B_r 
		\cdot  
		\EE_{s' \sim \cE_s^{\pi_1}} \Big[\norm[\big]{\pi_1(\cdot\given s') - \pi_2(\cdot\given s')}_1 \Big].
	\end{align}
	Here the inequality follows from $|Q^{\pi}(s, a)| \le (1-\gamma)^{-1} \cdot B_r$.
	We let $\cB=\cB'B_r(1-\gamma)^{-2}$. 
	 Plugging \eqref{eq:V-diff} into \eqref{eq:ubwr2}, we have 
	\begin{align*}
	W_2(\bar\rho_{\pi_1},\bar\rho_{\pi_2})\le \cB\cdot \sup_{s\in \cS} \EE_{s' \sim \cE_{s}^{\pi_1}} \Big[\norm[\big]{\pi_1(\cdot\given s') - \pi_2(\cdot\given s')}_1\Big].
	\end{align*}
Recall that we have $\rho_{\pi}=\bar{\rho}_\pi+(1-\alpha^{-1})(\rho_0-\bar{\rho}_\pi)$ in \eqref{eq:rho_pi}.
Then, by Lemma \ref{lem:w2-zoom}, it holds that
	\begin{align*}
		W_2(\rho_{\pi_1},\rho_{\pi_2})\le \alpha^{-\frac{1}{2}} W_2(\bar\rho_{\pi_1},\bar\rho_{\pi_2})\le \alpha^{-\frac{1}{2}}\cB\cdot \sup_{s\in \cS} \EE_{s' \sim \cE_{s}^{\pi_1}} \Big[\norm[\big]{\pi_1(\cdot\given s') - \pi_2(\cdot\given s')}_1\Big],
	\end{align*}
	which completes the proof of Lemma \ref{lem:full-RAVF}.
\end{proof}

\section{Technical Results}
In this section, we state and prove some technical results used in the proof of main theorems and lemmas.
\begin{lemma}
	\label{lem:boundness}
	Under Assumptions \ref{asp:mdp} and \ref{asp:nn}, it holds for any $\rho\in \sP_2(\RR^D)$ that
	\begin{align}\label{eq:lem-boundness1}
		\sup_{x\in \cX} \bigl| Q(x; \rho)\bigr| &\le \alpha\cdot B_1 \cdot W_2(\rho, \rho_0), \\
		\label{eq:lem-boundness2}
		\sup_{\theta\in \RR^D}\norm[\big]{\nabla_\theta g(\theta; \rho)}_{\rm F} &\le \alpha^{-1} \cdot B_2 \cdot\big( 2\alpha \cdot B_1 \cdot W_2(\rho, \rho_0) + B_r \big).
	\end{align}
\end{lemma}
\begin{proof}
	Following from Assumptions \ref{asp:mdp} and \ref{asp:nn}, we have that $\norm{\nabla_\theta\sigma(x; \theta)} \le B_1$ for any $x\in \cX$ and $\theta\in \RR^D$, which implies that $\lip(\sigma(x; \cdot)/B_1) \le 1$ for any $x \in \cX$. Note that $Q(x; \rho_0) = 0$ for any $x \in \cX$. Thus, by \eqref{eq:w1-dual} and the inequality between $W_1$ distance and $W_2$ distance \citep{villani2003topics}, we have for any  $ \rho\in \sP_2(\RR^D) $ and $x\in \cX$ that
	\begin{align}
		\label{eq:pf-bo1}
		\bigl| Q(x; \rho) \bigr| &= \alpha\cdot \biggl| \int \sigma(x; \theta) \cdot  \,\rd (\rho - \rho_0)(\theta) \biggr|  \le \alpha\cdot B_1 \cdot W_1(\rho, \rho_0) \le \alpha \cdot B_1 \cdot W_2(\rho, \rho_0).
	\end{align}
	which completes the proof of \eqref{eq:lem-boundness1} in Lemma \ref{lem:boundness}.
	Following from the definition of $g$ in \eqref{eq:g-rho}, we have for any $x\in \cX$ and $\rho\in \sP_2(\RR^D)$ that
	\begin{align*}
		\bigl\|\nabla_\theta g(\theta; \rho) \bigr\|_{\rm F} 
		&\le  \alpha^{-1} \cdot \EE_{\tilde \cD}\Bigl[ \bigl| Q(x; \rho) - r - \gamma\cdot Q(x'; \rho)\bigr|\cdot \norm[\big]{\nabla_{\theta\theta}^2 \sigma(x; \theta)}_{\rm F} \Bigr] \nonumber\\
		&\le \alpha^{-1} \cdot B_2 \cdot\big( 2\alpha \cdot B_1 \cdot W_2(\rho, \rho_0) + B_r \big).
	\end{align*} 
	Here the last inequality follows from \eqref{eq:pf-bo1} and the fact that $ \norm{\nabla_{\theta\theta}^2 \sigma(x; \theta)}_{\rm F} \le B_2 $ for any $x\in \cX$ and $\rho\in \sP_2(\RR^D)$, which follows from Assumptions \ref{asp:mdp} and \ref{asp:nn}. Thus, we complete the proof of Lemma \ref{lem:boundness}.
\end{proof}

\begin{lemma}\label{lem:w2-triangle}
	For $\rho_0, \rho_t, \rho_{\pi_t} \in \sP_2(\RR^D)$ and the geodesic $\alpha_t^{[0,1]}$ connecting $\rho_t$ and $\rho_{\pi_t}$, we have
	\begin{align}\label{eq:w2-triangle}
		\sup_{s\in[0,1]} W_2(\alpha_t^s,\rho_0) \le 2\max\{W_2(\rho_{\pi_t}, \rho_0), W_2(\rho_t,\rho_0)\}.
	\end{align}
	\begin{proof}
		We give a proof by contradiction.
		Note that $\alpha_t^s$ is the geodesic connecting $\rho_{\pi_t}$ and $\rho_t$.
		Assume there exists $t$ such that
		$$\sup_{s\in[0,1]} W_2(\alpha_t^s,\rho_0) > 2\max\{W_2(\rho_{\pi_t}, \rho_0), W_2(\rho_t,\rho_0)\}.$$
		Then, according to the triangle inequality of $W_2$ metric \citep{villani2008optimal}, we have 
		$$W_2(\rho_t,\alpha_t^s)\ge |W_2(\alpha_t^s,\rho_0)-W_2(\rho_t,\rho_0)|>|2W_2(\rho_t,\rho_0)-W_2(\rho_t,\rho_0)|=W_2(\rho_t,\rho_0),$$ 
		and $W_2(\rho_{\pi_t},\alpha_t^s)>W_2(\rho_{\pi_t},\rho_0)$
		for the same sake, which conflicts with the definition of geodesic that $W_2(\rho_t,\alpha_t^s)+W_2(\alpha_t^s, \rho_{\pi_t})=W_2(\rho_t,\rho_{\pi_t}) \le W_2(\rho_t,\rho_0)+W_2(\rho_{\pi_t},\rho_0)$.
		Hence, such $t$ does not exist and \eqref{eq:w2-triangle} holds.
		Thus we complete the proof of Lemma \ref{lem:w2-triangle}.
	\end{proof}
\end{lemma}

\begin{lemma}\label{lem:chi}
	The Chi-squared divergence has the following properties.
	\begin{itemize}
		\item[(i)]\label{lemit:coef} For any probability measure $g\in \sP(\Theta)$,  function $f:\Theta \rightarrow \RR$, and $\alpha\in \RR$, we have
		$$\chi^2(\alpha f \d| g) =\alpha^2 \chi^2(f \d| g)+(1-\alpha)^2.$$
		\item[(ii)]\label{lemit:times} For any probability measures $g_1\in \sP(\Theta_1)$ and  $g_2 \in \sP(\Theta_2)$, functions $f_1: \Theta_1 \rightarrow \RR$ and $ f_2:\Theta_2\rightarrow \RR$, we have
		$$\chi^2(f_1\times f_2 \d| g_1 \times g_2) =\chi^2(f_1 \d| g_1)\cdot \chi^2(f_2 \d| g_2)+\chi^2(f_1 \d| g_1) + \chi^2(f_2 \d| g_2),$$
		where $f_1\times f_2$ is the product of $f_1$ and $f_2$, and $g_1\times g_2$ is the product measure of $g_1$ and $g_2$, i.e., $(f_1\times f_2)(\theta_1\times\theta_2)=f_1(\theta_1)f_2(\theta_2)$ and $(g_1\times g_2)(\theta_1\times\theta_2)=g_1(\theta_1)g_2(\theta_2)$, respectively.
		\item[(iii)]\label{lemit:sum} For any probability measure $g \in \sP(\Theta)$, functions $f_1: \Theta\rightarrow \RR$ and $ f_2:\Theta\rightarrow \RR$, we have
		$$\chi^2(f_1 + f_2 \d| g)\le 3\big(\chi^2(f_1 \d| g)+\chi^2(f_2 \d| g)+1\big).$$
		\item[(iv)]\label{lemit:int} For any probability measure $g\in \sP(\Theta)$, function $f:\cX\times\Theta \rightarrow \RR$ and $\alpha:\cX\rightarrow \RR$, we have
		$$\chi^2\Bigl(\int \alpha(x) f(d,\cdot) dx \,\Big{\|}\, g\Bigr)\le
		\int \alpha(x)^2\rd x \cdot \int \chi^2(f(x,\cdot)\,\|\,g) \rd x + \Big(\int \alpha(x) \rd x -1\Big)^2.$$
	\end{itemize}
\end{lemma}
	\begin{proof}
		\noindent{\bf Proof of Property (i) of Lemma \ref{lem:chi}.} By definition of the Chi-squared divergence,
		we have 
		\begin{align*}
			\Cnorm{}{\alpha f}{g}&=\int \Big(\frac{\alpha f}{g}-1\Big)^2 \rd g\nend
			&=\int \bigg(\alpha\Big(\frac f g -1\Big)+\alpha-1\bigg)^2 \rd g\nend
			&=\alpha^2 \Cnorm{}{f}{g}+(\alpha-1)^2.
		\end{align*}
		Thus, we complete the proof of Property (i) of Lemma \ref{lem:chi}.
		\Vskip
		\noindent{\bf Proof of Property (ii) of Lemma \ref{lem:chi}.}
		Let $\tf_1=f_1/g_1-1$ and $\tf_2=f_2/g_2-1$. It then holds that
		\begin{align*}
			\chi^2(f_1\times f_2 \d| g_1 \times g_2)
			&=\int \Big(\frac{f_1\times f_2}{g_1\times g_2}-1\Big)^2 \rd(g_1\times g_2)
			=\int (\tf_1 \times \tf_2 + \tf_1 +\tf_2)^2 \rd(g_1\times g_2).
		\end{align*}
	By further noting that  $\int \tf_1 \rd g_1=0$, $\int \tf_2 \rd g_2=0$, $\int \tf_1^2 \rd g_1=\Cnorm{}{f_1}{g_1}$, and  $\int \tf_2^2 \rd g_2=\Cnorm{}{f_2}{g_2}$, we have
		\begin{align*}
			&\chi^2(f_1\times f_2 \d| g_1 \times g_2)\\
			&\quad= \int \big({\tf}_1^2 \times {\tf}_2^2 + {\tf}_1^2 + {\tf}_2^2+2(\tf_1)^2\times \tf_2+2(\tf_2)^2 \times \tf_1+2\tf_1 \times \tf_2\big) \rd(g_1 \times g_2)\\
			&\quad=\chi^2(f_1 \d| g_1)\cdot\chi^2(f_2 \d| g_2)+\chi^2(f_1 \d| g_1) + \chi^2(f_2 \d| g_2).
		\end{align*}
	Thus, we complete the proof of Property (ii) of Lemma \ref{lem:chi}.
		\Vskip
		\noindent{\bf Proof of Property (iii) of Lemma \ref{lem:chi}.}
		Let $\tf_1=f_1/g_1-1$ and $\tf_2=f_2/g_2-1$. It then holds that
		\begin{align*}
			\chi^2(f_1+f_2 \d| g)
			=\int \Big(\frac{f_1+f_2}{g}-1\Big)^2 \rd g
			=\int (\tf_1+\tf_2+1)^2 \rd g.
		\end{align*}
	By further noting that $\int \tf_1^2 \rd g_1=\Cnorm{}{f_1}{g_1}$ and  $\int \tf_2^2 \rd g_2=\Cnorm{}{f_2}{g_2}$, we have
		\begin{align*}
			\chi^2(f_1+f_2 \d| g)
			&\le 3 \int \big((\tf_1)^2 + (\tf_2)^2 +1\big) \rd g\\
			&= 3\big(\chi^2(f_1 \d| g)+\chi^2(f_2 \d| g)+1\big),
		\end{align*}
	where the inequality follows from the Cauchy-Schwarz inequality.
	Thus, we complete the proof of Property (iii) of Lemma \ref{lem:chi}.
		\Vskip
		\noindent{\bf Proof of Property (iv) of Lemma \ref{lem:chi}.}
		Let $\tf(x,\cdot)=f(x,\cdot)/g(\cdot)-1$. It then holds that
		\begin{align*}
			\chi^2\Big(\int \alpha(x) f(x,\cdot) \rd x \,\big\|\, g \Big)
			&=\int \Big(\int \alpha(x) \tf(x,\theta) \rd x\Big)^2 g(\rd\theta) +\Big(\int \alpha(x) \rd x -1\Big)^2\\
			&\le \int \Big( \int \alpha(x)^2 \rd x  \cdot \int {\tf}(x,\theta)^2\rd x\Big) g(\rd\theta)+\Big(\int \alpha(x) \rd x -1\Big)^2\\
			&=\int \alpha(x)^2\rd x \cdot\int \chi^2(f(x,\cdot)\d|g) \rd x + \Big(\int \alpha(x) \rd x -1\Big)^2,
		\end{align*}
	where the first equality holds by noting that $\int \tf(x,\cdot) \rd g=0$ and the inequality follows from the Cauchy-Schwarz inequality.
	Thus, we complete the proof of Property (iv) of Lemma \ref{lem:chi}.
	\end{proof}

\begin{lemma} \label{lem:G}
	For weighted homogeneous Sobolev norm defined by
	\begin{align}
		\Hnorm{\nu_1-\nu_2}{\mu}=\sup\big\{ |\inp{f}{\nu_1-\nu_2}| \big| \norm{f}_{\dot H^1(\mu)}\le 1\big\} \nonumber.
	\end{align}
	we have the following properties.
	\begin{itemize}
		\item[(i)] For a group of probability measures $\mu_x: \cX \rightarrow \sP_2(\Theta)$ and $\nu_1,\nu_2\in \sP_2(\Theta)$, if $\mu$ is in the convex hull of $\mu_x$, i.e., there exists $\alpha_x\ge0$ such that both $\int \alpha_x \rd x=1$ and $\mu=\int \alpha_x \mu_x \rd x$ hold, it then holds that
		\begin{align}
			\Hnorm{\nu_1-\nu_2}{\mu}\le \sup_{x\in \cX} \Hnorm{\nu_1-\nu_2}{\mu_x}\nonumber.
		\end{align}
		\item[(ii)] Assume that we have measures $\mu\in \sP_2(\Theta)$ and $\nu_x:\cX \rightarrow \sP_2(\Theta)$.
		Let $\beta_1, \beta_2: \cX \rightarrow \RR$ be two functions on $\cX$ such that $\int \beta_1(x) \rd x=\int \beta_2(x) \rd x$.
		Then, by letting $\nu_1=\int \nu_x \beta_1(x)\rd x$ and $\nu_2=\int \nu_x \beta_2(x)\rd x$, we have
		\begin{align*}
			\Hnorm{\nu_1-\nu_2}{\mu}\le \frac{1}{2} \sup_{(x',x'')\in \cX\times \cX} \Hnorm{\nu_{x'}-\nu_{x''}}{\mu}\cdot \int \big|\beta_1(x)-\beta_2(x)\big|\rd x.
		\end{align*}
	\end{itemize}
\end{lemma}
\begin{proof}
	\noindent{\bf Proof of Property (i) of Lemma \ref{lem:G}.}
	By definition of the weighted homogeneous Sobolev norm, we have
	\begin{align}
		\Hnorm{\nu_1-\nu_2}{\mu}
		&=\sup_f \bigg\{ \big|\inp{f}{\nu_1-\nu_2}\big| \bigg| \int |\nabla f|^2 \alpha_x \rd \mu_x \rd x \le 1\bigg\}\nend
		&=\sup_{f,\lambda_x\ge 0} \bigg\{ \big|\inp{f}{\nu_1-\nu_2}\big| \bigg| \int |\nabla f|^2 \rd \mu_x \le \lambda_x, \forall x;\int \alpha_x \lambda_x \rd x=1\bigg\} \label{eq:G-1}\\
		&\le \sup_{\substack{\int \lambda_x \alpha_x \rd x=1,\\ \lambda_x\ge 0}} \inf_x \sup_{f_x} \bigg\{ \big|\inp{f_x}{\nu_1-\nu_2}\big| \bigg| \int |\nabla f_x|^2 \rd \mu_x \le \lambda_x\bigg\} \label{eq:G-2},
	\end{align}
	where the first equality holds by noting that $\mu=\int \alpha_x \mu_x \rd x$.
	To illustrate the last equality, we denote by $\sF$ and $ \sF_x$ the allowed function class for $f$ in \eqref{eq:G-1} and $f_x$ in \eqref{eq:G-2} to choose from, respectively.
	Note that we have $\sF\subseteq \sF_x$ for any $x$, which is because the constraints for each $f_x$ in \eqref{eq:G-2} are relaxation of the constraints for $f$ in \eqref{eq:G-1}.
	And so, the supremum taken over $\sF$ is no larger than the supremum over $\sF_x$ for any $x$.
	Therefore, the supremum over $\sF$ is no larger than the the smallest supremum over $\sF_x$.
	Thus, \eqref{eq:G-2} holds.
	Furthermore, we have
	\begin{align}
		\Hnorm{\nu_1-\nu_2}{\mu}
		&\le \sup_{\substack{\int \lambda_x \alpha_x \rd x=1,\\ \lambda_x\ge 0}} \inf_x \sup_{f_x} \bigg\{ \big|\inp{f_x}{\nu_1-\nu_2}\big| \bigg| \int |\nabla f_x|^2 \rd \mu_x \le \lambda_x\bigg\} \label{eq:G-3}\\
		&= \sup_{\substack{\int \lambda_x \alpha_x \rd x=1,\\ \lambda_x\ge 0}} \inf_x \big\{\sqrt{\lambda_x} \Hnorm{\nu_1-\nu_2}{\mu_x} \big\}\label{eq:G-4}\\
		&\le \sup_{\substack{\int \lambda_x \alpha_x \rd x=1,\\ \lambda_x\ge 0}} \inf_x \big\{\sqrt{\lambda_x}  \big\}\cdot \sup_x \Hnorm{\nu_1-\nu_2}{\mu_x}\nonumber,
	\end{align}
	where the first equality holds by noting that \eqref{eq:G-4} is a rescaling of \eqref{eq:G-3} with respect to $\lambda_x$ in the constraints of \eqref{eq:G-3}. 
	Here, we let 
	$$y=\sup_{\substack{\int \lambda_x \alpha_x \rd x=1,\\ \lambda_x\ge 0}} \inf_x \big\{\sqrt{\lambda_x}  \big\}.$$
	Then, it holds that $y\le 1$.
	Otherwise, there must exists $\lambda_x$ such that $\int \lambda_x \alpha_x \rd x=1$ and that $\lambda_x > 1$ holds for any $x\in \cX$, which contradicts with our conditions that $\int \alpha_x \rd x=1$ and $\alpha_x \ge 0$.
	Therefore, it further holds that
	\begin{align*}
		\Hnorm{\nu_1-\nu_2}{\mu}\le y\cdot \sup_x \Hnorm{\nu_1-\nu_2}{\mu_x}\le \sup_x \Hnorm{\nu_1-\nu_2}{\mu_x}.
	\end{align*}
	Thus, we complete the proof of Property (i) of Lemma \ref{lem:G}.
	\Vskip
	\noindent{\bf Proof of Property (ii) of Lemma \ref{lem:G}.}
	Let $\alpha=\beta_1-\beta_2$. Then, we have $\int \alpha( x)\rd x=0$. Let $\alpha^+=\max\{0, \alpha\}$ and $\alpha^{-}=-\min\{0,\alpha\}$. Then, we have $\alpha^+-\alpha^-=\alpha=\beta_1-\beta_2$ and that $\alpha^+ +\alpha^-=|\alpha|=|\beta_1-\beta_2|$.
	Since $\int \alpha=0$, it holds that $\int \alpha^+(x)\rd x=\int \alpha^-(x)\rd x=A$, where $A\ge 0$. We further let $\lambda(x,x')=A^{-1}\alpha^+(x)\alpha^-(x')$. Then, it holds that $\alpha^+(x)=\int \lambda(x,x')\rd x'$ and that $\alpha^-(x')=\int \lambda(x,x')\rd x$. Therefore, we have
	\begin{align*}
		\Hnorm{\nu_1-\nu_2}{\mu}=\Hnorm[\Big]{\int_\cX (\alpha^+-\alpha^-)\nu_x \rd x}{\mu}=\Hnorm[\Big]{\int_{\cX\times\cX} \lambda(x,x')(\nu_x-\nu_{x'})\rd x\rd x'}{\mu}.
	\end{align*}
	By defintion of the weighted homogeneous Sobolev norm, it further holds that
	\begin{align}\label{eq:G-5}
		&\Hnorm[\Big]{\int_{\cX\times\cX} \lambda(x,x')(\nu_x-\nu_{x'})\rd x\rd x'}{\mu}\nend
		&\quad =\sup_{f} \bigg\{ \Big| \inp[\big]{\int \lambda(x,x')(\nu_x-\nu_{x'})\rd x\rd x'}{f}\Big| \,\bigg|\, \int |\nabla f|^2 \le 1 \bigg\}. 
	\end{align}
	We assume the supremum in \eqref{eq:G-5} is reached at $f^*$. Then, we have $\int \normf{\nabla f^*}^2\rd \mu\le 1$ and that
	\begin{align*}
		\Hnorm[\Big]{\int_{\cX\times\cX} \lambda(x,x')(\nu_x-\nu_{x'})\rd x\rd x'}{\mu}
		&=\Big| \inp[\big]{\int \lambda(x,x')(\nu_x-\nu_{x'})\rd x\rd x'}{f^*}\Big| \nend
		& \le \sup_{(x,x')\in \cX\times \cX} \normf[\big]{\inp{\nu_x-\nu_{x'}}{f^*}} \cdot \int \lambda( x, x') \rd x \rd x' \nonumber\\
		&\le \frac{1}{2} \sup_{(x,x')\in \cX\times \cX} \Hnorm{\nu_x-\nu_{x'}}{\mu}\cdot \int |\beta_1-\beta_2| \rd x,
	\end{align*}
	where the last inequality holds by noting that $\int \lambda(x, x')\rd x\rd x'=A=1/2\cdot\int |\alpha| \rd x=1/2\cdot\int |\beta_1-\beta_2|\rd x$.
	Thus, we complete the proof of Property (ii) of Lemma \ref{lem:G}.
\end{proof}

\begin{lemma}\label{lem:w-ind}
	For probability measures $\nu_1,\nu_2\in \sP_2(\Theta_1)$ and $p_1, p_2\in \sP_2(\Theta_2)$, it holds that
	\begin{align*}
		W_2^2( \nu_1 \times p_1, \nu_2 \times p_2) \le W_2^2 (\nu_1, \nu_2) + W_2^2 (p_1, p_2).
	\end{align*}
\end{lemma}
\begin{proof}
By the property of optimal transport \citep{ambrosio2008gradient}, there exists mapping $T_\nu$ and $T_p$ such that
\begin{align*}
	W_2^2(\nu_1,\nu_2)=\int \norm{\theta_1 -T_\nu(\theta_1)}  \nu_1(\rd\theta_1),\\
	W_2^2(p_1,p_2)=\int \norm{\theta_1 -T_p(\theta_1)}  p_1(\rd\theta_1),
\end{align*}
and that $(T_\nu)_\sharp \nu_1=\nu_2$ and $(T_p)_\sharp p_1=p_2$.
Note that we have $(T_\nu \times T_p)_\sharp (\nu_1\times p_1)=\nu_2 \times p_2$.
Thus, by definition of $W_2$ distance, it holds that 
\begin{align*}
	W_2^2( \nu_1 \times p_1, \nu_2 \times p_2) 
	&\le \int \norm[\big]{(\theta_1,\theta_2)-\big(T_\nu(\theta_1),T_p(\theta_1)\big)}^2\nu_1(\rd \theta_1)p_1(\rd \theta_2)\\
	&=\int \Big( \norm[\big]{\theta_1-T_\nu(\theta_1)}^2+\norm[\big]{\theta_2-T_p(\theta_2)}^2 \Big) \nu_1(\rd \theta_1)p_1(\rd \theta_2)\\
	&=W_2^2 (\nu_1, \nu_2) + W_2^2 (p_1, p_2).
\end{align*}
Thus, we complete the proof of Lemma \ref{lem:w-ind}.
\end{proof}


\begin{lemma} \label{lem:w2-zoom}
	For probability density function $\rho,\rho_1,\rho_2\in \sP_2(\Theta)$, let $\tilde{\rho}_1=\alpha^{-1}\rho_1+(1-\alpha^{-1})\rho$ and $   \tilde{\rho}_2=\alpha^{-1}\rho_2+(1-\alpha^{-1})\rho$. Then, We have
$$W_2(\tilde{\rho}_1,\tilde{\rho}_2)\le \alpha^{-1/2} W_2(\rho_1,\rho_2).$$

	\begin{proof}
		 Recall the definition of Wasserstain-2 distance that
		 \begin{align*}
		     W_2(\rho_1,\rho_2)=\Big(\inf_{\gamma\in \Gamma(\rho_1,\rho_2)}\int \norm{x-y}^2\rd\gamma(x,y)\Big)^{1/2},
		 \end{align*}
	 	 where $\Gamma(\rho_1,\rho_2)$ is the set of all couplings of $\rho_1$ and $\rho_2$.
		 We assume that the infimum is reached by $\gamma^*(x,y)\in \Gamma(\rho_1,\rho_2)$, i.e., 
		 \begin{align*}
		 	W_2(\rho_1,\rho_2)=\Big(\int \norm{x-y}^2\rd\gamma^*(x,y)\Big)^{1/2}.
		 \end{align*}
		 We denote by $\gamma'(x,y)$ the distribution such that
		 \begin{align*}
		     \gamma'(x,y)=\alpha^{-1}\gamma^*(x,y)+(1-\alpha^{-1})\rho(x)\delta(x-y),
		 \end{align*}
		 where $\delta(x,y)$ is dirac delta function and it holds that  $\gamma'(x,y)\in \Gamma(\tilde{\rho}_1,\tilde{\rho}_2)$.
		 Hence, it follows that
		 \begin{align*}
		     W_2(\tilde{\rho}_1,\tilde{\rho}_2)
		     &\le \Big(\int \norm{x-y}^2\rd \gamma'(x,y)\Big)^{1/2}\\
		     &=\Big( \int \norm{x-y}^2 \big(\alpha^{-1}\gamma^*(x,y)+(1-\alpha^{-1})\rho(x)\delta(x-y)\big)\rd (x,y)\Big)^{1/2}\\
		     &=\alpha^{-1/2}W_2(\rho_1,\rho_2).
		 \end{align*}
	 Thus, we complete the proof of Lemma \ref{lem:w2-zoom}.
	\end{proof}
\end{lemma}



\section{Auxiliary Lemmas}
We use the definition of absolutely continuous curves in $\sP_2(\RR^D)$ in \cite{ambrosio2008gradient}.
\begin{definition}[Absolutely Continuous Curve] \label{def:ac-curve}
	Let $\beta: [a, b] \rightarrow \sP_2(\RR^D)$ be a curve. Then, we say $\beta$ is an absolutely continuous curve if there exists a square-integrable function $f:[a, b] \rightarrow \RR$ such that
	\begin{align*}
		W_2( \beta_s, \beta_t) \le \int_{s}^{t} f(\tau) \,\rd \tau
	\end{align*}
	for any $a\le s< t\le b$.	
\end{definition}

Then, we have the following first variation formula.
\begin{lemma}[First Variation Formula, Theorem 8.4.7 in \cite{ambrosio2008gradient}]
	\label{lem:diff}
	Given $\nu \in \sP_2(\RR^D)$ and an absolutely continuous curve $\mu: [0, T] \rightarrow \sP_2(\RR^D)$, let $\beta: [0,1] \rightarrow \sP_2(\RR^D)$ be the geodesic connecting $\mu_t$ and $\nu$. It holds that 
	\begin{align*}
		\frac{\rd }{\rd t}\frac{ W_2(\mu_t, \nu)^2}{2} = -\inp{\dot\mu_t}{\dot\beta_0}_{\mu_t,W_2},
	\end{align*}
	where $\dot\mu_t = \partial_t \mu_t$, $\dot\beta_0 = \partial_s \beta_s\given_{s=0}$, and the inner product is defined in \eqref{eq:w-inner}.
\end{lemma}

\begin{lemma}[Eulerian Representation of Geodesics, Proposition 5.38 in \cite{villani2003topics}]
	\label{lem:euler}
	Let $\beta: [0, 1] \rightarrow \sP_2(\RR^D)$ be a geodesic and $v$ be the corresponding vector field such that $\partial_t \beta_t = - \Div(\beta_t \cdot v_t)$. It holds that
	\begin{align*}
		\frac{\rd}{\rd t}(\beta_t \cdot v_t) = - \Div(\beta_t \cdot v_t\otimes v_t),
	\end{align*}
	where $\otimes$ is the outer product between two vectors.
\end{lemma}

\begin{lemma}[Talagrand's Inequality, Corollary 2.1 in \cite{otto2000generalization}]
	\label{lem:talagrand}
	Let $\nu$ be $N(0, \kappa \cdot I_D)$. It holds for any $\mu \in \sP_2(\RR^D)$ that
	\begin{align*}
		W_2(\mu, \nu)^2 \le 2 D_{\rm KL}(\mu \,\|\, \nu) / \kappa.
	\end{align*}
\end{lemma}

\begin{lemma}[Theorem 1 in \cite{peyre2011comparison}]\label{lem:w2-ub}
	Let $\mu,\nu$ be two probability measures in $\sP_2(\theta)$. Then, it holds that
	\begin{align*}
		W_2(\mu, \nu) \le 2\norm{\mu-\nu}_{\dot H^{-1} (\nu)}.
	\end{align*}
\end{lemma}

%% file: Limitation.tex

\section{Conclusions and Limitations}
In this work, we study the time envolution of a two-timescale AC represented by a two-layer neural network in the mean-field limit.
Specifically, the actor updates its policy via proximal policy optimization, which is closely related to the replicator dynamics, while the critic updates by temporal-difference learning, which is captured by a semigradient flow in the Wasserstein space. 
By introducing a restarting mechanism, we establish the convergence and optimality of AC with two-layer overparameterized neural network.
However, the study has potential limitations.
In this work we only study the continuous-time limiting regime, which is an idealized setting with infinitesimal learning rates, and establish finite-time convergence and optimality guarantees. Finite-time results for the more realistic discrete-time setting is left for future research.